\documentclass{article}

\usepackage[preprint]{neurips_2026}

\usepackage[utf8]{inputenc} % allow utf-8 input
\usepackage[T1]{fontenc}    % use 8-bit T1 fonts
\usepackage{hyperref}       % hyperlinks
\usepackage{url}            % simple URL typesetting
\usepackage{booktabs}       % professional-quality tables
\usepackage{amsfonts}       % blackboard math symbols
\usepackage{nicefrac}       % compact symbols for 1/2, etc.
\usepackage{microtype}      % microtypography
\usepackage{xcolor}         % colors
\usepackage{amsmath}
\usepackage{amssymb}
\usepackage{amsthm}
\usepackage{amsfonts}
\usepackage{mathrsfs}
\usepackage{mathtools}
\usepackage[ruled,vlined]{algorithm2e}
\usepackage{subcaption}
\usepackage[rightcaption]{sidecap}

\numberwithin{equation}{section} 
\newtheorem{theorem}{Theorem}[section]
\newtheorem{definition}{Definition}[section]
\newtheorem{lemma}{Lemma}[section]

\DeclareMathOperator{\Ad}{\mathcal{A}{\it d}}
\DeclareMathOperator{\GL}{GL}

\newcommand{\g}{\mathfrak{g}}
\newcommand{\gl}{\mathfrak{gl}}

\newcommand{\R}{\mathbb{R}}

\newcommand{\Lind}{\mathscr{L}}
\newcommand{\Param}{\mathcal{P}}
\newcommand{\Lmat}{\mathcal{L}}
\title{Planning Neural Dynamics with Lie Group Embedding\\ through Supervised Projective Manifold Learning}
\author{%
  \setcounter{footnote}{2}
  Tianwei Wang\textsuperscript{$\dagger$}\thanks{Bio-inspired Computing and Machine Learning (BCML) Lab} \\
  School of Informatics\\
  University of Edinburgh\\
  \And
  Bryan Chen\\
  School of Mathematics \\
  University of Edinburgh\\
  \And
  Qian Zuo\\
  School of Informatics \\
  University of Edinburgh\\
  \And
  Qiyue Xia\\
  School of Informatics \\
  University of Edinburgh\\
  \And
  Xin Li\\
  School of Computer Science\\
  Beijing Institute of Technology\\
  \And
  \setcounter{footnote}{1}
  Wei Pang\textsuperscript{$\ddagger$}\thanks{Corresponding to \texttt{w.pang@hw.ac.uk} and \texttt{t.wang-110@sms.ed.ac.uk}}\\
  School of MACS\\
  Heriot-Watt University\\
}

\begin{document}
\maketitle
\begin{abstract}
We propose Lie group embedded dynamical neural networks (LieEDNN) and the corresponding learning algorithms based on gradient descent and metric projection on smooth manifold, where we treat Lie group as an intrinsic representation for continuous symmetry of manifold geometry. Thereby we achieve learnable and stable dynamics on the underlying manifold for general Lie group, and we are able to utilize the powerful representation capability of Lie group such as $\mathrm{SO}(3)$ and $\mathrm{SE}(3)$ to solve real world engineering problems in areas such as robotics, graphics, and control. Two core challenges are: (i) General Lie groups are incompatible with addition arithmetic, which is necessary for neural network interactions. (ii) The dynamics evolve in the nonlinear representation space of special algebra rather than the normal Euclidean space, which violates the paradigm of common neural ODEs. To address these two challenges, we firstly introduce adjoint Lie group action on the Lie algebra, which induces a linear mapping and transfer to the block-wise structure of weight matrices, such that addition could operate on the Lie algebra as a vector space. Then we parameterize the Lie algebra and the adjoint action as linear transformation so that the architecture is aligned with neural network perceptrons. Explicitly, this embedding appears as block-wise manifold constraints on weights, and we develop algorithms to learn the equilibrium with stability guarantees of the temporal neural network dynamics. Experiments are implemented on a specific Lie group $\mathrm{SE}(3)$, with the application scenario of telescopic manipulators.
\end{abstract}

\section{Introduction}
Neural dynamical systems formulate computation as the temporal evolution of a continuous-time system, where the desired output is encoded as a stable equilibrium of the dynamics. This viewpoint connects recurrent neural networks (RNN) \cite{RNN}, continuous Hopfield networks (CHNN) \cite{Hopfield:2007}, residual networks (ResNets) \cite{ResNets}, and neural ODEs \cite{neuralODE}, and it is particularly suitable for tasks where stability, convergence, and trajectory generation are central requirements. However, most existing neural dynamical architectures are formulated in Euclidean spaces, while many physical systems are governed by intrinsic geometric structures. For example, robotic motion, rigid-body transformation, camera pose, and articulated mechanisms are more naturally described by Lie groups \cite{LieOrigin,LieTheory} such as \(\mathrm{SO}(3)\) and \(\mathrm{SE}(3)\), rather than by plain Euclidean coordinates \cite{Robot1, Robot2}.

The difficulty of algebraic embedding is that Lie group generally represents a nonlinear manifold and does not admit closed addition required by the weighted summation operation in standard neural network paradigms \cite{lecun2015deep,Deeplearning,BackPropagation}. Directly placing neural states or weights on a Lie group therefore conflicts with the additive interaction in neural networks. LieEDNN addresses this issue by transferring the dynamics from the Lie group to its Lie algebra, where addition is well defined since Lie algebra is a vector space \cite{LieTheory}, and by using the adjoint action of Def. \ref{def:adjoint action} to induce linear representations of Lie group $G$ elements on the Lie algebra $\mathfrak{g}$. Moreover, with retraction mapping $e^X$, where $X\in\mathfrak{g}$, we can learn the desired dynamics on the Lie group manifold \cite{IntroductionToSmoothManifolds, LieTheory}. In this way, Lie group elements appear as structured recurrent weight blocks, while the dynamics of neural states remain compatible with ordinary neural network analysis including matrix multiplication, weighted sum, and learning rules.

In this paper, we develop Lie group embedded dynamical neural networks (LieEDNN) with supervised equilibrium learning and projective manifold constraints for structure preservation. The target state is learned as a stable equilibrium, and weight periodic projection in Alg.~\ref{alg.1} preserves the induced Lie group block-wise structure in the weight matrix during training. The framework extends quaternion-valued Hopfield-structured neural dynamics \cite{qshnn} from a specific hypercomplex non-commutative algebra to a more general class of objects, namely matrix Lie group. We instantiate the method on \(\mathrm{SE}(3)\) and demonstrate its application for telescopic manipulator planning in Sec.~\ref{sec:experiment}, where rotational and translational coupling can be encoded directly through neural connections.

\section{Background}
\label{backgrounds}

\paragraph{Matrix Lie group and manifold}Manifold is a geometric object $\mathcal{M}$ that looks locally like a piece of $\mathbb{R}^n$. A smooth manifold is a manifold $\mathcal{M}$ together with a collection of local coordinates covering $\mathcal{M}$ such that the change of coordinates map between two overlapping coordinate systems is smooth, called a diffeomorphism \cite{IntroductionToSmoothManifolds,DifferentialManifolds}. A Lie group is a smooth manifold equipped with a group structure such that the operations of group multiplication and group inversion are smooth \cite{LieTheory}. As the terminology suggests, a matrix Lie group is a Lie group, e.g. $\mathrm{SU}(2)$, $\mathrm{GL}(n;\mathbb{C})$. In this paper, when we refer to a Lie group, we also mean its underlying smooth manifold. Def. \ref{def:LieGroup} presents definition of Lie group. For matrix Lie group, group product is matrix multiplication.

\begin{definition}[Lie Group \cite{LieTheory,LieOrigin}]A Lie group is a smooth manifold G which also admits group structure and the group product $G\times G\rightarrow G$ and the inverse map $G\rightarrow G$ are smooth.
\label{def:LieGroup}
\end{definition}
There are also several specific examples of matrix Lie group and corresponding Lie algebra we use in this paper defined below in Def. \ref{def:SO(3)andO(3)andso(3)ando(3)} and Def. \ref{def:se(3) and SE(3)}. For the orthogonal group $O(3)$, we have the column vectors are orthogonal to each other. Let $R=[R_1, R_2, R_3]\in O(3)$, where $R\in\mathbb{R}^{3\times 3}, R_{1,2,3}\in\mathbb{R}^3$, then $R_i\cdot R_j=\delta_{ij}$. This condition is equivalent to $R^{\mathrm{T}}=R^{-1}$. Calculate the determinant of both sides, $\mathrm{det}(R)=\mathrm{det}(R^{\mathrm{T}})=\mathrm{det}(R)^{-1}$, then $\mathrm{det}({R})^2=1$. For $\mathrm{det}(R)=1$, we call it special orthogonal matrices, which represent rotation transformation in 3-dimensional space. For $\mathrm{det}(R)=-1$, it represents rotation transformation with a reflection compound, which transforms left hand to right hand, keeping the geometric details the same except for chirality.
\begin{definition}[Orthogonal Group and Special Orthogonal Group \cite{groupsandsymmetry}] We define:
\begin{equation}
    \mathrm{O}(3):=\left\{\left.R\in\mathbb{R}^{3\times 3}\right|R^\mathrm{T}=R^{-1}\right\},\quad
    \mathrm{SO}(3):=\left\{\left.R\in\mathbb{R}^{3\times 3}\right|R\in\mathrm{O}(3),\mathrm{det(R)=1}\right\}
\end{equation}
with matrix multiplication as group product. The Lie algebras of $\mathrm{O}(3)$ and $\mathrm{SO}(3)$ are the same set:
\begin{equation}
    {\mathfrak{o}(3)}=\mathcal{T}_e\mathrm{O}(3)\cong{\mathfrak{so}(3)}=\left\{\left.X\in\mathbb{R}^{3\times 3}\right|X^\mathrm{T}=-X\right\}=\mathcal{T}_e\mathrm{SO}(3)
    \label{eq:so(3)}
\end{equation}
with matrix commutator $[X,Y]:=XY-YX$ defined as Lie bracket operation on the Lie algebra. 
\label{def:SO(3)andO(3)andso(3)ando(3)}
\end{definition}

\begin{definition}[Skew-symmetric Representation \cite{groupsandsymmetry}]A skew-symmetric matrix is defined by $X^\mathrm{T}=-X,\ X\in\mathbb{R}^{3\times 3}$, with parameterization $\boldsymbol{t}\in\mathbb{R}^3$ and notation $X=[\boldsymbol{t}]^\wedge\in\mathbb{R}^{3\times 3}$. 
\begin{equation}
[\mathbf{t}]^\wedge:=
\left[\begin{matrix}
0 & -t_3 & t_2\\
t_3 & 0 & -t_1\\
-t_2 & t_1 & 0
\end{matrix}\right]\in\mathfrak{so}(3),\quad \boldsymbol{t}=\left[\begin{matrix}
t_1  \\
t_2  \\
t_3 
\end{matrix}\right]\in\mathbb{R}^3
\label{eq:ss}
\end{equation}
The explicit form of skew-symmetric representation and parameterization vector is Eq. (\ref{eq:ss}). And the linear mapping of $[\mathbf{t}]^\wedge$ on arbitrary vector $\boldsymbol{v}\in\mathbb{R}^3$ equals to cross product $[\mathbf{t}]^\wedge\boldsymbol{v}=\boldsymbol{t}\times \boldsymbol{v}$.
\label{def:Skew-symmetric Representation}
\end{definition}

\begin{definition}[Special Euclidean group $\mathrm{SE}(3)$ and $\mathfrak{se}(3)$ \cite{groupsandsymmetry}]\hspace{-0.1cm}They are defined by
\begin{equation}
    \mathrm{SE}(3):=\left\{\left.\left[\begin{matrix}
R & \boldsymbol{p}\\
O_{1\times 3} & 1 
\end{matrix}\right]\right|R\in\mathrm{SO}(3),\boldsymbol{p}\in\mathbb{R}^3\right\},\quad
{\mathfrak{se}(3)}=\left\{\left.\left[\begin{matrix}
[\boldsymbol{\omega}]^\wedge & \boldsymbol{v}\\
O_{1\times 3} & 0 
\end{matrix}\right]\right|\boldsymbol{\omega},\boldsymbol{v}\in\mathbb{R}^3\right\}
\end{equation}
\label{def:se(3) and SE(3)}with matrix multiplication as group product and commutator $[X,Y]:=XY-YX$ as Lie bracket operation. $\mathfrak{se}(3)$ expands the tangent space of $\mathrm{SE}(3)$ at identity element $\mathbb{I}_3$ \cite{LieTheory,LieOrigin}.
\end{definition}
Special Euclidean group represents all isometric transformations in 3-dimensional space including rotation and translation \cite{LieTheory}. Its Lie algebra quantifies the angular and translational momentum. It is used to model the motion of joints chain and reference coordinates of telescopic manipulators. Our major aim of this paper is to plan temporal neural dynamics on $\mathrm{SE}(3)$ to embody the framework of general Lie group embedding. Later we deduce another matrix group in Eq. (\ref{eq:L[Ad]}). The following Thm. \ref{thm:closed subgroup theorem} gives how to judge a matrix group is a Lie group, known as the closed subgroup theorem.
\begin{theorem}[Closed Subgroup \cite{ClosedSubgroupTheorem2,ClosedSubgroupTheoremCartan1930}]
    \label{thm:closed subgroup theorem}Suppose $G$ is a Lie group, $H$ is a subgroup of $G$. If $H$ is topologically closed in $G$, then $H$ is also a Lie group with embedding smooth manifold structure. Specifically, this holds for matrix Lie group.
\end{theorem}
\vspace*{-0.3cm}
\paragraph{Lie group and Lie algebra}A Lie group and its Lie algebra is a correspondence. Consider the smooth manifold of a Lie group $G$, at the identity element $e$, the tangent space $\mathcal{T}_{e}G$ is a vector space. And Lie algebra is defined on this vectors set $\mathfrak{g}$. Besides addition operation, there is another closed operation $[\cdot,\cdot]:\mathfrak{g}\times\mathfrak{g}\rightarrow\mathfrak{g}$ called Lie bracket which satisfying $[X,Y]=-[Y,X]$ and Jacobi identity $[X,[Y,Z]]+[Y,[X,Z]]+[Z,[X,Y]]\equiv0$. For the Lie algebra of matrix Lie group, Lie bracket is defined by the commutator $[X,Y]:=XY-YX$ \cite{LieTheory}. Geometrically, Lie algebra represents the differential of Lie group action on its underlying smooth manifold, and the matrix exponential $\mathrm{Exp}(X)\!:=\!e^X$ retracts the points from $\mathfrak{g}$ to the underlying manifold $G$ \cite{IntroductionToSmoothManifolds}. Take Def. \ref{def:SO(3)andO(3)andso(3)ando(3)} as an example. Consider a curve family in $\mathrm{SO}(3)$, denoted by $\eta(t)$, satisfying $\eta(0)=\mathbb{I}$. Since $\eta(t)\eta(t)^{\mathrm{T}}=\mathbb{I}_3$, by taking time derivative, we have $\dot{\eta}(t)\eta(t)^{\mathrm{T}}+\eta(t)\dot{\eta}(t)^{\mathrm{T}}=O_{3\times 3}$. Let $t=0$, then $\dot{\eta}(0)=X\in\mathfrak{so}(3)$ and $\dot{\eta}(t)+\dot{\eta}(t)^{\mathrm{T}}=O_{3\times 3}$. Thus, we verify that $X^\mathrm{T}=-X$ as in Eq. (\ref{eq:so(3)}). 

\paragraph{Dynamical neural networks} Neural dynamics in this paper refers to the following architecture stated in Eq. (\ref{eq:DNN}), where we treat the network evolution as a differential autonomous dynamical system \cite{neuralODE,DynamicalNN,qshnn}. In vector form, Eq. (\ref{eq:DNN}) is defined as a ODE system: $\frac{\text{d}{}}{\text{d}t}\boldsymbol{x}=-\gamma\boldsymbol{x}+\mu W\phi(\boldsymbol{x})+\mu\boldsymbol{b}$, where connection weights are denoted by $W\in\mathbb{R}^{N\times N}$, neural state variables are denoted by $\boldsymbol{x}\in\mathbb{R}^N\!\times\!\mathcal{T}$, and biases are denoted by $\boldsymbol{b}\in\mathbb{R}^N$. The activation is $\phi(\cdot)$, typically chosen as a hyperbolic tangent function applied entry-wise on vector variables.
\begin{align}
&\frac{\text{d}{}}{\text{d}t}x_j(t)=-\gamma x_j(t)+\mu\sum_{i=1}^Nw_{ji}\phi\left[x_i(t)\right]+\mu b_j,\ \ \ \ \ j=1,2,...n
\label{eq:DNN}
\end{align}
where $\gamma$ and $\mu$ are constants. It can be viewed as a continuous Hopfield neural network (HNN) \cite{Hopfield:2007,hnn} in structure.  Unlike classic HNN, we do not impose symmetric connections, extra constraints on weight matrix are needed to achieve global stability \cite{delay,qshnn,asymmetricHNN}, and the equilibrium can also be controlled by supervised learning \cite{qshnn}. Further, when we simulate the temporal neural system Eq. (\ref{eq:DNN}) with asymptotic stability, it converges to the stationary point under numerical precision with a finite simulation time range. Hence the architecture can also be treated as continuous residual networks (ResNets) \cite{ResNets,neuralODE}, and it is analogous to a deep ResNets with identical layers. Further, it also lies in the category of generalized recurrent neural networks with fully connected topology.

\section{Methodology}\label{sec:methodology}
\subsection{Embedding Representation of Lie Group}
\paragraph{Conjugate mapping and adjoint action}As the network system in Eq. (\ref{eq:DNN}) suggests, we activate a neuron by summing up the weighted stimulation from all other neurons interacting with it. To retain geometric meaning of the matrix Lie group on these interactions, we must first make the interaction compatible with addition. However, Lie group is nonlinear as a manifold, it is impossible to find a general linear representation where addition is well defined. Our strategy is to transfer the objects of neural states explanation from Lie group to its Lie algebra and define the adjoint action on it for the neuron interactions through Lie group weights embedding. Since the adjoint action is linear, there exists a linear representation matrix to calculate the action on Lie algebra parameterized vector, which forms a new group $\mathcal{L}[\mathcal{A}d_G]$ as the practical representation of Lie group action in LieEDNN.
\vspace{0.2cm}
\begin{definition}[Conjugate Mapping on Lie group]For Matrix Lie group $G$ and a fixed element $g\in G$, the conjugate mapping is defined by $C_g\!:\!G\to G,\ C_g(h):=ghg^{-1}$. $C_g$ is a group automorphism, where $C_{g}\circ C_h=C_{gh}$ and $C_g^{-1}=C_{g^{-1}}$. The group product is compatible with matrix multiplication.
\label{def:conjugate mapping}
\end{definition}
\vspace{-0.2cm}
\begin{definition}[Adjoint Action on Lie Algebra]Consider the matrix Lie group action on its Lie algebra, with a fixed $g\in G$, by left action of $g$ and right action of $g^{-1}$, the compound action forms adjoint action as in Eq. (\ref{eq:adjoint action}) and illustrated in Fig. \ref{fig:LieTheory}. Operations $\circ$ are matrix multiplication.
\begin{equation}
    \mathcal{A}d_g(X):=g \circ\! X\!\circ g^{-1},\quad g\in G, \ X\in\mathfrak{g}.
\label{eq:adjoint action}
\end{equation}
\label{def:adjoint action}
\end{definition}
\vspace{-0.5cm}
\begin{figure}[ht]
\hspace*{0.1cm}
\begin{subfigure}[t]{0.58\textwidth}
    \centering
    \includegraphics[width=\textwidth]{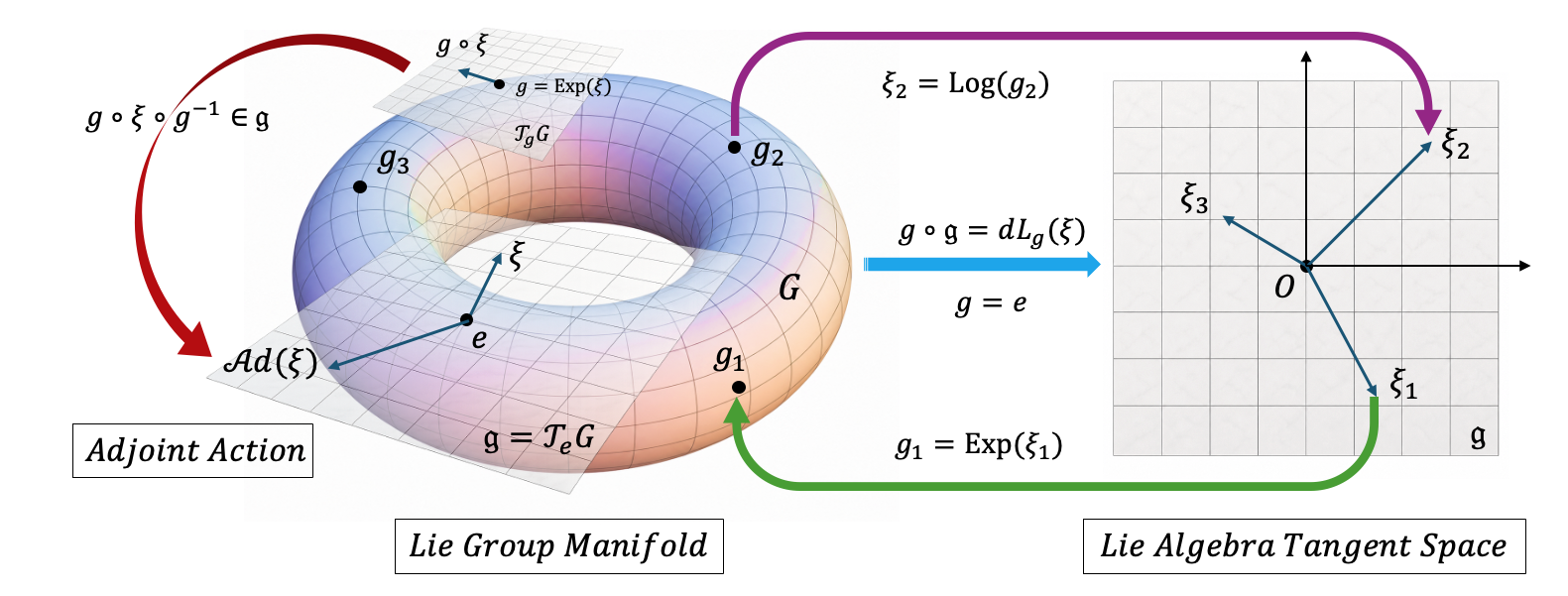}
    \label{fig:LieTheorya}
  \end{subfigure}
  %\hspace*{0.2cm}
  \begin{subfigure}[t]{0.36\textwidth}
    \centering
    \includegraphics[width=\textwidth]{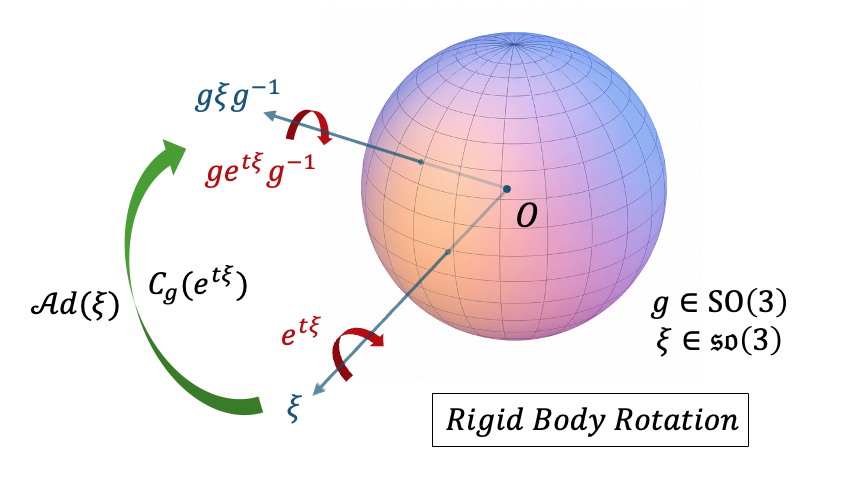}
    \label{fig:rigidBody}
  \end{subfigure}
  \vspace*{-0.3cm}
  \caption{Through conjugate Lie group action $\circ$ with $g$ and $g^{-1}$, we define the adjoint action whereby linearize a Lie group action on Lie algebra. Matrix exponential $\mathrm{Exp}(\cdot)$ and logarithm $\mathrm{Log}(\cdot)$ allow us to retract Lie algebra $\mathfrak{g}$ to Lie group manifold $G$, vice versa. For $\mathrm{SO}(3)$ and $\mathfrak{so}(3)$, adjoint action $\mathcal{A}d(\cdot)$  maps two rigid body rotations which is a linear transformation on $\mathfrak{so}(3)$ by Thm.~\ref{thm:LieAlgebraAutomorphism}.}
  \label{fig:LieTheory}
\end{figure}

\begin{theorem}[Adjoint Action as Automorphism]For a fixed $g\in G$, with the adjoint action defined by Def. 
\ref{def:adjoint action}, $\Ad_g(\cdot)$ is a Lie algebra automorphism on $\g$, i.e. $\Ad_g:\mathfrak{g}\rightarrow\mathfrak{g}$ is isomorphism. 
\label{thm:LieAlgebraAutomorphism}
\end{theorem}
{\it\textbf{Proof:}} Consider a curve by retracting exponential function \cite{LieTheory,IntroductionToSmoothManifolds}, $\gamma(t)\!:=\!e^{tX}$ satisfying $\gamma(0)\!=\!e$, where $X\in \mathfrak{g}$ . Then, $\gamma(t)$ is a curve evolves in Lie group over $t\in\mathbb{R}$.
\begin{equation}
    \left.\dot{\gamma}(t)\right|_{t=0}=\left.\frac{\text{d}}{\text{d}t} e^{tX}\right|_{t=0}=\left.Xe^{tX}\right|_{t=0}=X\in\mathcal{T}_eG
\end{equation}
which is to say that $X$, the generator of the Lie group curve $e^{tX}$, is tangent to the curve at $t\!=\!0,\ \gamma\!=\!e$. Next, we use conjugate mapping of Def. \ref{def:conjugate mapping} on the entire curve $\gamma(t)$ with an arbitrary fixed element $g\in G$. This mapping defines a new curve $g\gamma(t) g^{-1}$ in Lie group manifold, denoted by $\zeta(t)$. Notice that $\zeta(0)=geg^{-1}=e$, i.e. $\zeta(t)$ crosses the identity $e$. Then we calculate the differential of $\zeta(t)$ at $t\!=\!0$ that should still be in the tangent space $\mathcal{T}_eG=\mathfrak{g}$. This leads to the adjoint action in Def .\ref{def:adjoint action}:
\begin{equation}
    \left.\frac{\text{d}}{\text{d}t}C_g[\gamma(t)]\right|_{t=0}=\left.\dot{\zeta}(t)\right|_{t=0}=\left.\frac{\text{d}}{\text{d}t} ge^{tX}g^{-1}\right|_{t=0}=\left.gXe^{tX}g^{-1}\right|_{t=0}=gXg^{-1}\in\mathfrak{g}
\end{equation}
Thus, the range of adjoint action mapping $\Ad_g(\cdot)=g(\cdot)g^{-1}$ is still $\mathfrak{g}$. Moreover, we know that $\forall Y\in\mathfrak{g}$, the origin exists by $g^{-1}Yg$, thus $\Ad(X)\!=\!Y$. $\Ad_g$ is a bijection. Further we have:
\begin{equation*}
\operatorname{Ad}_g([X,Y])=g(XY-YX)g^{-1}=(gXg^{-1})(gYg^{-1})-(gYg^{-1})(gXg^{-1})=[\operatorname{Ad}_g(X),\operatorname{Ad}_g(Y)]
\end{equation*}
Then adjoint action preserves bracket operation and becomes an automorphism of Lie algebra $\mathfrak{g}$.\hfill\(\square\)

\begin{definition}[Induced Linear Mapping of $\mathrm{SE}(3)$ by Adjoint Action]\label{def:L}As the adjoint action of $\mathrm{SE}(3)$ is linear on Lie algebra $\mathfrak{se}(3)$, it naturally induces a linear homomorphism by $\mathcal{L}$:
\begin{align}
    \mathcal{L}{[}Ad_{\mathrm{SE}(3)}{]}=\left\{\left.\left[\begin{matrix}
    R&O_{3\times 3}\\
    [\boldsymbol{t}]^{\wedge}R&R\end{matrix}\right]\right|R^T=R^{-1},\boldsymbol{t}\in\mathbb{R}^3\right\}\subset\mathbb{R}^{6\times6}
    \label{eq:L[Ad]}
\end{align}
\label{def:Induced Linear Mapping by Adjoint Action}
Adjoint action induced mapping exists for general Lie algebra given by Thm. \ref{thm:General Linear Representation of Adjoint Action} in App. \ref{D}.
\end{definition}

\begin{theorem}[Induced Linear Mapping of $\mathrm{SE}(3)$ Forms Lie Group]$\mathcal{L}{[}Ad_{\mathrm{SE}(3)}{]}$ by Def. \ref{def:L} is topologically closed in $\mathrm{GL}(6,\mathbb{R})$, which forms a Lie group and an underlying manifold defined with matrix representation. Further, the group operation is compatible with matrix multiplication.
\label{thm:Induced Linear Mapping as a Lie Group Manifold}
\end{theorem}
{\it\textbf{Proof:}} Thm. \ref{thm:Induced Linear Mapping as a Lie Group Manifold} is proved in App. \ref{B}. Theorem for general Lie group is presented in App. \ref{E}. \hfill$\square$

\begin{theorem}[Linear Representation of Adjoint Action on $\mathfrak{se}(3)$]For Lie algebra $\mathfrak{se}(3)$ defined in Def. \ref{def:se(3) and SE(3)}, each element is parameterized by $\boldsymbol{\omega}$ and $\boldsymbol{v}$, we denote the vector parameterization of $X\in\mathfrak{se}(3)$ by $P(X):=[\boldsymbol{\boldsymbol{\omega}}\ \boldsymbol{v} ]^{\mathrm{T}}\in\mathbb{R}^6$, then there exists a linear mapping $\mathcal{L}[{\mathcal{A}d_{\mathrm{SE}(3)}}]\in\mathcal{L}(\mathbb{R}^6)$ induced by the adjoint action of $\mathrm{SE}(3)$ on $\mathfrak{se}(3)$ of the form in Def. \ref{def:Induced Linear Mapping by Adjoint Action}. Further, we have
\begin{equation}
    P(\mathcal{A}d_g(X))=\mathcal{L}[\mathcal{A}d_g]\circ P(X)
    \label{eq: Linear Representation of Adjoint Action}
\end{equation}
i.e. the linear operator on parameterized $\mathfrak{se}(3)$ vector is equivalent to parameterizing the adjoint action $\mathcal{A}d_g(\cdot)$ of a fixed $g\in \mathrm{SE}(3)$ on $X\in\mathfrak{se}(3)$.
\label{thm:Linear Representation of Adjoing Action}
\end{theorem}
\textit{\textbf{Proof:}} Thm. \ref{thm:Linear Representation of Adjoing Action} is proved in App. \ref{C}. Theorem for general Lie group is in App. \ref{D}\hfill$\square$

\subsection{Network Architecture for LieEDNN}\label{sec:architecture}
\paragraph{Lie group embedded dynamical neural networks}To construct a straight embedding of Lie group within network Eq. \ref{eq:DNN}, we utilize Lie algebra and the linearity of adjoint action of Def. \ref{def:adjoint action}. The dynamics is defined on Lie algebra, with variables $X_i(t)\in\mathfrak{g}$ over time $t\in\mathbb{R}$. And the neural interactions are through adjoint action, which is represented by weights in Lie group $\theta_{ij}\in G$. The bias should also be Lie algebra elements $B_i\in\mathfrak{g}$ to have a consistent addition to the system. With notations upon, the system equation appears as Eq. (\ref{eq:LieEDNN}) below:
\begin{equation}
    \frac{\text{d}{}}{\text{d}t}{X}_i=-\gamma X_i+\mu\sum_j \alpha_{ij}\mathcal{A}d_{\theta_{ij}}[\phi(X_j)]+\mu B_i,\quad i=1,2,...,N
    \label{eq:LieEDNN}
\end{equation}
where $\gamma$ is a constant factor controlling resonance decay, and $\mu$ is constant to scale interaction level. Function $\phi$ is chosen to be a hyperbolic tangent $\phi(x)=\frac{e^x-e^{-x}}{e^x+e^{-x}}$. There is an extra difference of Eq. \ref{eq:LieEDNN} from the classic architecture Eq. \ref{eq:DNN}. We add a connection strength factor $\alpha_{ij}\in\mathbb{R}$ on representative interactions. These are treated as learnable parameters to solve the non-convergence problem of learning Lie group weights. This is discussed later with more details on network learning.

\begin{figure}[ht]
\begin{subfigure}[t]{0.6\textwidth}
    \centering
    \includegraphics[width=\textwidth]{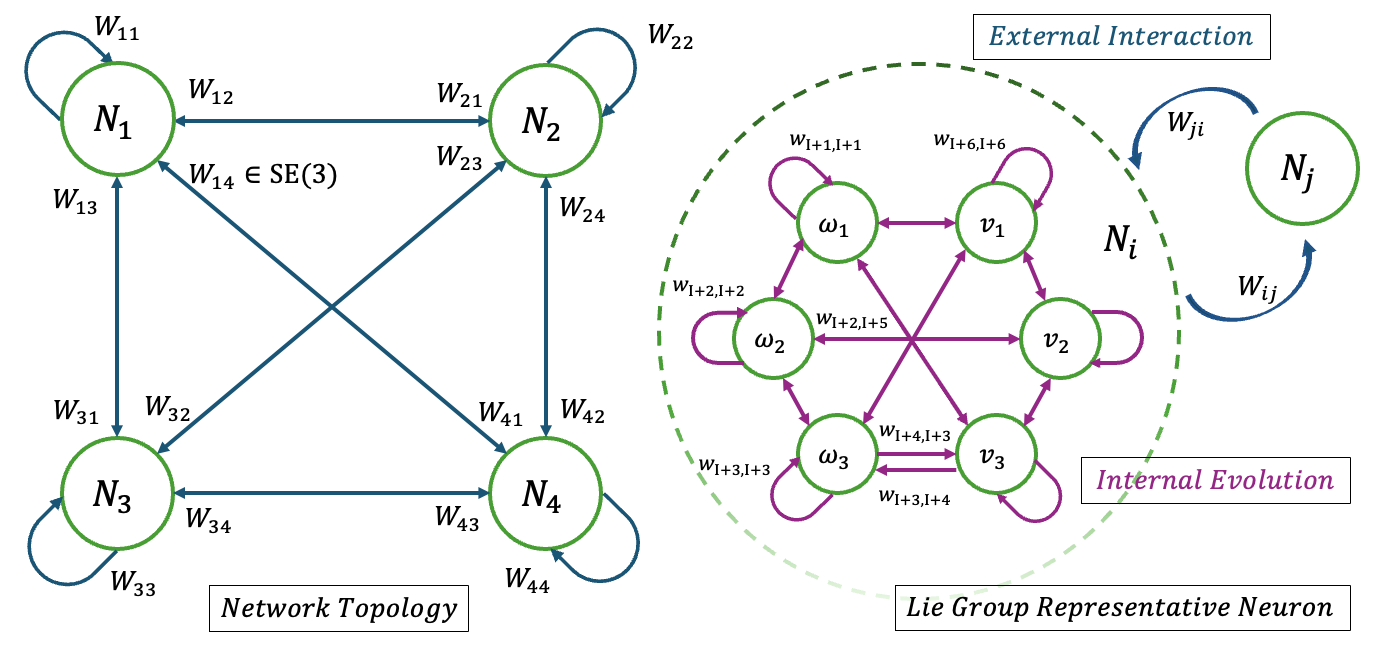}
    \caption{Topology and Block Structure of Lie Group Representation}
    \label{fig:network}
  \end{subfigure}
  \hspace*{0.2cm}
  \begin{subfigure}[t]{0.32\textwidth}
    \centering
    \includegraphics[width=\textwidth]{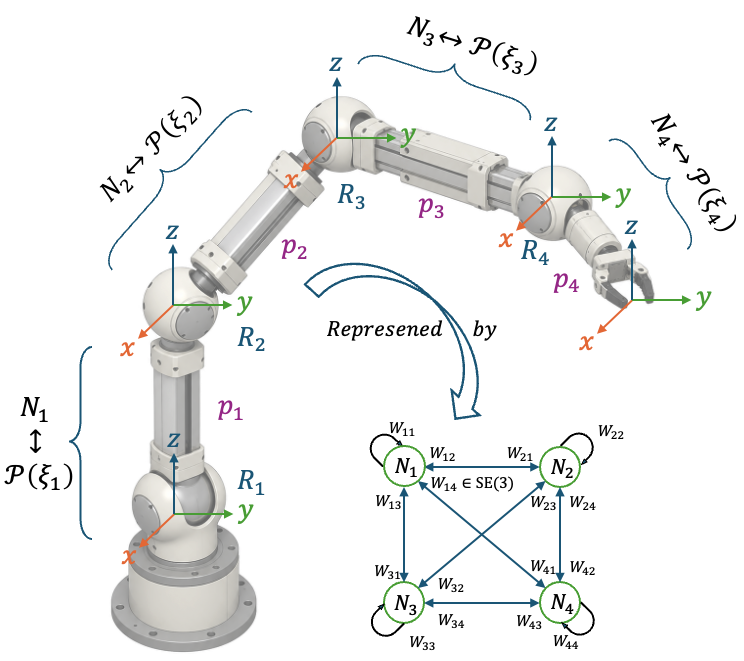}
    \caption{Robotic Application Scenario}
    \label{fig:arm}
  \end{subfigure}
  \caption{(a). Each Lie group neuron is integrated by several $\mathbb{R}$ neurons with number same as the dimension of Lie group underlying manifold. (b). Manipulator with translational and rotational joints. A telescopic robotic manipulator with four compound joints can be represented by LieEDNN with $\mathrm{SE}(3)$ and $\mathfrak{se}(3)$ in Eq. (\ref{eq:gradient2}) and Eq. (\ref{eq:LieEDNN3}). The neural states encode the posture of the manipulator, and the neural dynamics is decoded as motion trajectories by App. \ref{L}. In Sec. \ref{sec:experiment}, we implement a prototype for this application module. Notations in the figure is aligned with the following paragraph.}
  \label{fig:Lienetwork}
\end{figure}

\paragraph{Representative dynamics}Note that the network equation we propose so far in Eq. (\ref{eq:LieEDNN}) is not a normal ODE in Euclidean space, as the neural states are matrix representations of Lie algebra and the dynamics also evolves on Lie algebra. This is inconvenient for us to analyze equilibrium, stability, and smoothness of trajectory, etc, with the theory of dynamical systems. As a vector space, $\mathfrak{g}$ allows us to define a parameterization by vector space $\mathbb{R}^d$, where $d$ is the dimension of Lie group manifold and the Lie algebra. Here is how we implement this transformation with parameterization mapping $\mathcal{P}$: Take $\mathrm{SE}(3)$ and $\mathfrak{se}(3)$ as example. By $\boldsymbol{\xi}_i:=\mathcal{P}(X_i)\in\mathbb{R}^6,\ \boldsymbol{b}_i:=\mathcal{P}(B_i)\in\mathbb{R}^6,\ X_i,B_i\in\mathfrak{se}(3)$, system equation Eq. (\ref{eq:LieEDNN}) is equivalent to Eq. \ref{eq:LieEDNN2} by applying identical parameterizations:
\begin{equation}
    \frac{\text{d}{}}{\text{d}t}{\boldsymbol{\xi}}_i=-\gamma \boldsymbol{\xi}_i+\mu\sum_j \alpha_{ij}\mathcal{L}[\mathcal{A}d_{\theta_{ij}}]\phi(\boldsymbol{\xi_j})+\mu \boldsymbol{b_i},\quad\boldsymbol{\xi}_i\in\mathbb{R}^6,\quad i=1,2,...,N
    \label{eq:LieEDNN2}
\end{equation}
In the form as ODEs on $\mathbb{R}^{6N}$, where $W\!:=\!\{w_{ij}\}\in\mathbb{R}^{6N\times 6N}$, satisfying $W_{ij}\!=\!\alpha_{ij}\mathcal{L}[\mathcal{A}d_{\theta_{ij}}]\in\mathbb{R}^{6\times 6}$, and $\ \alpha_{ij}\in\mathbb{R}$, the system equation equivalent to Eq. (\ref{eq:LieEDNN2}) is $\frac{\text{d}{}}{\text{d}t}\boldsymbol{\xi}=-\gamma\boldsymbol{\xi}+\mu W\phi(\boldsymbol{\xi})+\mu\boldsymbol{b}$, i.e.
\begin{equation}
\frac{\text{d}{}}{\text{d}t}{\xi}_i=-\gamma{\xi}_i+\mu \sum_jw_{ij}\phi({\xi_j})+\mu{b}_i,\quad \xi_i\in\mathbb{R},\quad i=1,2,\dots6N
\label{eq:LieEDNN3}
\end{equation}
where $ \boldsymbol{\xi}:=[\xi_1,\xi_2,\dots,\xi_{6N}]^{\mathrm{T}}$. The index of the ODE dynamics is $1\le i\le6N$ since the dimension of $\mathrm{SE}(3)$ is $6$. The index of the representative dynamics on Lie algebra is $1\le i\le N$ where $N$ is the number of neurons. Every $6$ entries of $\boldsymbol{\xi}$ are combined as one $\boldsymbol{\xi}_i$, and the index transformation is $\boldsymbol{\xi}_i=[\xi_{I+1},\xi_{I+2},\dots\xi_{I+6}]^\mathrm{T}$, where $I:=6(i-1)$. We call $i,j$ block indices, $I,J$ plain indices since the representations of Lie algebra in Eq. (\ref{eq:LieEDNN2}) appear as vector blocks in Eq. (\ref{eq:LieEDNN3}) with dimension $6$ for the case of $\mathfrak{se}(3)$. Fig. \ref{fig:network} demonstrates the network topology with respect to neural dynamics in Eq. (\ref{eq:LieEDNN2}) and Eq. (\ref{eq:LieEDNN3}). Representative embedding blocks of the explicit form in equations can be found in App. \ref{K}. Properties of the system Eq. (\ref{eq:LieEDNN3}) are provided by the following theorem.

\begin{theorem}[Equilibrium Condition and Asymptotic Stability]\label{Thm:EandU}For the differential dynamical system of LieEDNN in Eq. (\ref{eq:LieEDNN3}), the existence and uniqueness of equilibrium are sufficient to be proven by following inequality on weight parameters. $L_i$ is the Lipschitz constant of activation $\phi_i(\cdot)$.
\begin{equation}
\sum_{i=1}^{6N}L_i|w_{ij}|<\frac{\gamma}{\mu}
\label{eq:weight constraint 1}
\end{equation}
The equilibrium that is asymptotically stable can be established by the inequality Eq. (\ref{eq:weight constraint 2}). $\mu$ and $\gamma$ are constants consistent with former notations. Eq. (\ref{eq:weight constraint 1}) and Eq. (\ref{eq:weight constraint 2}) can be satisfied by constraint on norm $\|W\|_1<\gamma/\mu$ and $\|W\|_\infty<\gamma/\mu$ at the normalization process at learning stage.
\begin{equation}
    \sum_{i=1}^{6N}(|w_{ji}|+|w_{ij}|)<\frac{2\gamma}{\mu}
    \label{eq:weight constraint 2}
\end{equation}
\end{theorem}
{\it\textbf{Proof:}} Thm. \ref{Thm:EandU} is proved in App.~\ref{G} and App.~\ref{I}. It also applies for architecture stated in Eq. (\ref{eq:DNN}).\hfill\(\square\)

With constraints on weights, we establish the existence and uniqueness of global equilibrium, as well as the asymptotic stability of the equilibrium without which a stable stationary output is not guaranteed for the network. This distinguishes LieEDNN from traditional energy-based model in two aspects. First, the equilibrium is global and learnable with algorithm in Sec. \ref{sec:experiment}. Second, the connection weights are not symmetric like Hopfield neural network, instead, we apply specific constraints on the norm of weight matrix as in Thm. \ref{Thm:EandU} by Eq. (\ref{eq:weight constraint 1}) and Eq. (\ref{eq:weight constraint 2}). Analogous conditions are explored in \cite{Asymmetic,asymmetric2,qshnn,asymmetricHNN}.

\subsection{Learning Rules on Synaptic Weights and Strength}\label{sec:learning}
\paragraph{Gradient direction for desired equilibrium}The optimization target of learning by square root error is $E:=\|\boldsymbol{\xi}^*-\boldsymbol{\xi}_d \|_2$. We try to learn the equilibrium of dynamical system Eq. (\ref{eq:LieEDNN3}). At current stage, we aim to calculate the Euclidean gradient, which means the structure on geometric representation of Lie group on the weight matrix is not taken into account. And the equilibrium of dynamical system Eq. (\ref{eq:LieEDNN3}), denoted by $\boldsymbol{\xi}^*$, satisfies the sensitivity equation:
\begin{equation}
    \gamma\boldsymbol{\xi}^\ast=\mu W\phi(\boldsymbol{\xi}^*)+\mu\boldsymbol{b}
    \label{eq:sensitivity equation}
\end{equation}
Take the derivative of $w_{ij}$ on Eq. (\ref{eq:sensitivity equation}), for the left hand side, ${\partial_{w _{i j}}}\mathrm{LHS}=\gamma{\partial_{w _{i j}}}\boldsymbol{\xi}^*$. On the right hand side, we adopt tensor convention for matrix as $(1,1)$ tensor, and vector as $(1,0)$ tensor. By chain rule on derivative and compound rule, with Einstein summation convention, the $k$th entry of $\mathrm{RHS}$ is $\frac{\partial}{\partial w _{i j}} w _{l}^k \phi({\boldsymbol{\xi}}^{*})^l= w _{l}^k \frac{\partial }{\partial w _{i j}} \phi ( \boldsymbol{\xi} ^{*})^l +  \delta _{i}^k \delta _{j l} \phi (\boldsymbol{\xi}^*)^l=w _{l}^k \phi '\! \left( {\xi}^*_l \right)\frac{\partial {\xi}^*_l}{\partial w _{ij}} + \delta _{i}^k \phi\!\left({\xi}^*_j \right)$. Therefore,
\begin{equation}
    \frac{\partial}{\partial w _{i j}}\mathrm{RHS}=\frac{\partial}{\partial w _{i j}}\mu W\phi(\boldsymbol{\xi}^*)+\frac{\partial}{\partial w _{i j}}\mu\boldsymbol{b}=\mu W\cdot\frac{\partial \phi(\boldsymbol{\xi}^*)}{\partial \boldsymbol{\xi}^*}\cdot\frac{\partial \boldsymbol{\xi}^*}{\partial w _{i j}}+\mu\boldsymbol{e}_i\phi({\xi}^*_j)
\end{equation}
We denote $\boldsymbol{e}_i$ the standard basis of $\mathbb{R}^{6N}$, with $i$th entry to be one and any others to be zero. Combining both sides of the derivatives for Eq. (\ref{eq:sensitivity equation}), we have $\frac{\gamma}{\mu}\frac{\partial \boldsymbol{\xi}^*}{\partial w _{i j}}=W\cdot\frac{\partial \phi(\boldsymbol{\xi}^*)}{\partial \boldsymbol{\xi}^*}\cdot\frac{\partial \boldsymbol{\xi}^*}{\partial w _{i j}}+\boldsymbol{e}_i\phi({\xi}^*_j)$.
Thus, the gradient of equilibrium about weights for the sensitivity equation Eq. (\ref{eq:sensitivity equation}) is
\begin{equation}
    \frac{\partial \boldsymbol{\xi}^*}{\partial w _{i j}}=\frac{\mu}{\gamma}\cdot\left[\mathbb{I}_{6N}-\frac{\mu}{\gamma}W\cdot J_{\boldsymbol{\phi}}(\boldsymbol{\xi}^*)\right]^{-1}\!\!\cdot\phi(\xi_j^*)\,\boldsymbol{e}_i:=\frac{\mu\phi(\xi_j^*)}{\gamma}S^{-1}\boldsymbol{e}_i
    \label{eq:derivative}
\end{equation}
We define the difference $\boldsymbol{\delta}:=\boldsymbol{\xi}^*-\boldsymbol{\xi}_d$ and sensitivity matrix $S:=\mathbb{I}_{6N}-\frac{\mu}{\gamma}W\cdot J_{\boldsymbol{\phi}}(\boldsymbol{\xi}^*)$ and substitute it in Eq. (\ref{eq:derivative}). By chain rule, the gradient direction of $E$ about network weights for Eq. (\ref{eq:sensitivity equation}) is
\begin{align}
\frac{\partial E}{\partial w_{ij}}=\frac{\partial E}{\partial(\boldsymbol{\xi}^*-\boldsymbol{\xi}_d)}\cdot\frac{\partial(\boldsymbol{\xi}^*-\boldsymbol{\xi}_d)}{\partial\boldsymbol{\xi}^*}\cdot\frac{\partial \boldsymbol{\xi}^*}{\partial w_{ij}}=\frac{(\boldsymbol{\xi}^*-\boldsymbol{\xi}_d)}{E}\cdot\frac{\partial \boldsymbol{\xi}^*}{\partial w_{ij}}=\frac{\mu\phi(\xi_j)}{\gamma E}\boldsymbol{\delta}^{\mathrm{T}}S^{-1}\boldsymbol{e}_i
\label{eq:gradient}
\end{align}

%%%%%%%%%%%%%%%%%%%%%%%%%%%%%%%%%%%%%%%%%%%%%%%%%%%%%%%%%%%%%%%%%%%%%%%%%%%
\paragraph{Gradient direction of synaptic strength $\boldsymbol{\alpha}$}We treat connection strength $\boldsymbol{\alpha}$, introduced in Eq. (\ref{eq:LieEDNN}), to be learnable parameters. With the same optimization target $E$. Recall that each block of \(W\) is of the form $W_{ij}=\alpha_{ij}\mathcal{L}[\mathcal{A}d_{\theta_{ij}}]$, where \(W_{ij}\in\mathbb{R}^{6\times 6}\) is the \((i,j)\)th block of the weight matrix \(W\in\mathbb{R}^{6N\times 6N}\). For \(a,b=1,2,\dots,6\), the corresponding entry of row $a$, column $b$ of the block $W_{ij}$ is denoted by $(W_{ij})_{ab}=\alpha_{ij}(\mathcal{L}[\mathcal{A}d_{\theta_{ij}}])_{ab}$. By the chain rule of tensor differential, we can calculate:
\begin{align}
    \frac{\partial \boldsymbol{\xi}^*}{\partial \alpha_{ij}}&=
    \frac{\partial \boldsymbol{\xi}^*}{\partial W_{ij}}:
    \frac{\partial W_{ij}}{\partial \alpha_{ij}}
    =\sum_{a,b=1}^6\frac{\partial \boldsymbol{\xi}^*}{\partial \bigl(W_{ij}\bigr)_{ab}}\
    \frac{\partial \bigl(W_{ij}\bigr)_{ab}}{\partial \alpha_{ij}}=\sum_{a,b=1}^6\frac{\partial \boldsymbol{\xi}^*}{\partial \bigl(W_{ij}\bigr)_{ab}}\mathcal{L}[\mathcal{A}d_{\theta_{ij}}]_{ab}
    \label{eq:middle gradient of alpha}
\end{align}
where operation $:$ is double contraction of tensors. We define the block indices as the convention used in Eq. (\ref{eq:LieEDNN3}) by $I=:6(i-1),\ J:=6(j-1)$. Using the formula of the equilibrium gradient with respect to weights $w_{ij}$ from Eq. (\ref{eq:derivative}), we have $\frac{\partial \boldsymbol{\xi}^*}{\partial (W_{ij})_{ab}}=\frac{\mu}{\gamma}S^{-1}\phi(\xi^{*}_{J+b})\boldsymbol{e}_{I+a}$, where \(\boldsymbol{e}_{I+a}\) denotes the basis vector corresponding to the \(a\)th entry of the \(i\)th block, i.e. entry with index $6(i-1)+a$. Moreover, as $\alpha_{ij}$ only affects block $W_{ij}$ by $\frac{\partial (W_{ij})_{ab}}{\partial \alpha_{ij}}=\mathcal{L}[\mathcal{A}d_{\theta_{ij}}]_{ab}=\frac{w_{I+a,J+b}}{\alpha_{ij}}$, by the chain rule of $\alpha_{ij}\in\mathbb{R}$, the gradient direction of the equilibrium error with respect to \(\alpha_{ij}\) is calculated:
\begin{align}
\frac{\partial E}{\partial \alpha_{ij}}
&=
\frac{\partial E}{\partial(\boldsymbol{\xi}^*-\boldsymbol{\xi}_d)}
\cdot
\frac{\partial(\boldsymbol{\xi}^*-\boldsymbol{\xi}_d)}{\partial\boldsymbol{\xi}^*}
\cdot
\frac{\partial \boldsymbol{\xi}^*}{\partial \alpha_{ij}}
=\sum_{a,b=1}^{6}\frac{\mu\phi(\xi^{*}_{J+b})}{\gamma E}\,
\boldsymbol{\delta}^{\mathrm T}S^{-1}
\mathcal{L}[\mathcal{A}d_{\theta_{ij}}]_{ab}\,
\boldsymbol{e}_{I+a}
\label{eq:gradient2}
\end{align}
As we already obtained the gradient on $w_{ij}$ from Eq. (\ref{eq:gradient}), substituting Eq. (\ref{eq:middle gradient of alpha}) to Eq. (\ref{eq:gradient2}), it is reduced to the following to calculate the gradient of connection strength $\alpha_{ij}$:
\begin{equation}
    \frac{\partial E}{\partial \alpha_{ij}}
    =\sum_{a,b=1}^{6}
    \frac{\partial E}{\partial \bigl(W_{ij}\bigr)_{ab}}
    \frac{\partial (W_{ij})_{ab}}{\partial \alpha_{ij}}=\sum_{a,b=1}^{6}\frac{\partial E}{\partial w_{I+a,J+b}}\frac{w_{I+a,J+b}}{\alpha_{ij}}
    \label{eq:gradient3}
\end{equation}
We use Eq. (\ref{eq:gradient}) with manifold projection and Eq. (\ref{eq:gradient3}) directly for the algorithm of network learning. Since geometric constraint is imposed on $W$, for $\mathrm{SE}(3)$, we develop approaches as follows.
\paragraph{Retraction to representation manifold}To apply Lie group embedding, as we develop the representative dynamics in Eq. (\ref{eq:LieEDNN2}), block-wise structure of $\mathcal{L}[\Ad_g]$ in Eq. (\ref{eq:L[Ad]}) should be imposed on weight matrix $W\in\mathbb{R}^{6N\times 6N}$ of Eq. (\ref{eq:LieEDNN3}). However, the gradient flow in Euclidean space does not preserve such manifold structure automatically. Our strategy is to conduct metric projection from the weights to the representative manifold periodically with a certain number of gradient descent iterations. This approach is essentially a first-order approximation of Riemann gradient \cite{Riemann1,Riemann2}, and its effectiveness is supported by empirical results in Sec. \ref{sec:experiment}. 

\begin{lemma}[Metric Projection to \(\mathrm{SO}(3)\)]
Given an arbitrary matrix \(\Pi\in\mathbb{R}^{3\times 3}\), the metric projection from \(\Pi\) to \(\mathrm{SO}(3)\) under the Frobenius norm is as follows:
\begin{equation}
\mathrm{Proj}_{\mathrm{SO}(3)}(\Pi)=U\,\mathrm{diag}\bigl(1,1,\det(UV^{\mathrm T})\bigr)V^{\mathrm T}
\end{equation}
where \(\Pi=U\Sigma V^{\mathrm T}\) is the singular value decomposition (SVD) of \(\Pi\), with \(\Sigma=\mathrm{diag}(\sigma_1,\sigma_2,\sigma_3)\) and singular values \(\sigma_1\ge\sigma_2\ge\sigma_3\ge0\). Equivalently, $\mathrm{Proj}_{\mathrm{SO}(3)}(\Pi)=\operatorname*{arg\,min}_{R\in\mathrm{SO}(3)}\|\Pi-R\|_F$.
\label{lem:metric projection}
\end{lemma}
{\it\textbf{Proof:}} Lem. \ref{lem:metric projection} is proved in App. \ref{A}.\hfill$\square$

\begin{theorem}[Metric Projection to Manifold $\mathscr{L}:=\mathcal{L}{[}Ad_{\mathrm{SE}(3)}{]}$ Block-wisely] Given an arbitrary matrix $\Theta\in\mathbb{R}^{6\times 6}$ with $3\times 3$ block division $\mathcal{A},\mathcal{B},\mathcal{C},\mathcal{D}\in\mathbb{R}^{3\times 3}$, the block-wise projection from $\Theta$ to $\mathscr{L}$ under the metric of Frobenius norm is calculated by the following formula:
\begin{align}
    \mathrm{Proj}_\mathscr{L}(\Theta)=\left[\begin{matrix}
    \mathrm{Proj}_{\mathrm{SO}(3)}(\mathcal{M})&O_{3\times 3}\\[0.2cm]
    [\boldsymbol{t_\Theta}]^\wedge\mathrm{Proj}_{\mathrm{SO}(3)}(\mathcal{M})&\mathrm{Proj}_{\mathrm{SO}(3)}(\mathcal{M})
    \end{matrix}\right]\in\mathscr{L}
    \label{eq:metric projection}
\end{align}
where we denote the arithmetic average of $\mathcal{A},\ \mathcal{B}$ and the skew-symmetry part of $\mathcal{C}R^\mathrm{T}$ as follows:
\begin{equation}
    \Theta:=\left[\begin{matrix}
    \mathcal{A}&\mathcal{B}\\[0.1cm]
    \mathcal{C}&\mathcal{D}
    \end{matrix}\right],\quad \mathcal{M}:=\frac{\mathcal{A}+\mathcal{D}}{2},\quad [\boldsymbol{t}_\Theta]^\wedge:=\frac{\mathcal{C}R^\mathrm{T}-R\ \!\mathcal{C^\mathrm{T}}}{2},\quad R:=\mathrm{Proj}_{\mathrm{SO}(3)}(\mathcal{M})
    \label{eq:metric projection2}
\end{equation}
\label{thm:Metric Projection to Manifold L}
\end{theorem}
{\it\textbf{Proof:}} Thm. \ref{thm:Metric Projection to Manifold L} is proved in App. \ref{F}. \hfill\(\square\)

\section{Experiments}\label{sec:experiment}
\paragraph{Experimental Design and Results}We evaluate LieEDNN on a prototype of planning module for a telescopic manipulator represented by four $\mathrm{SE}(3)$ representative neurons shown in Fig. \ref{fig:arm}. The desired state \(\boldsymbol{\xi}_d\in\mathbb R^{6N}\) is sampled in the Lie algebra coordinate space, and each block is decoded to a local \(\mathrm{SE}(3)\) pose through exponential retraction in App. \ref{L} for robot control. The experiment studies whether the learned autonomous dynamics can drive the network equilibrium to \(\boldsymbol{\xi}_d\), while preserving the adjoint-induced block structure of weights through periodic manifold projection. We compare different projection periods $\mathscr{P}$ and include a no-projection ablation, where the same Euclidean gradient update is used but the manifold projection is disabled. The results report convergence loss Fig. \ref{fig:loss}, accuracy Fig. \ref{fig:accuracy}, final weight structures Fig. \ref{fig:heatmap1}, and the evolution of the learned equilibrium.
\paragraph{Learning Algorithm for LieEDNN}The algorithm Alg. \ref{alg.1} in App. \ref{M} provides the learning process of LieEDNN and the choices of hyper-parameters for the experiments, where the approaches in Sec. \ref{sec:methodology} are implemented. Periodic projection algorithm is essentially a first-order approximation of Riemann gradient \cite{Riemann1,Riemann2}, where descent curves of the loss are perturbed but still oscillating towards convergence, as shown in Fig. \ref{fig:loss}. 

\begin{figure}[h]
\centering
  \begin{subfigure}[t]{0.24\textwidth}
    \centering
    \includegraphics[width=\textwidth]{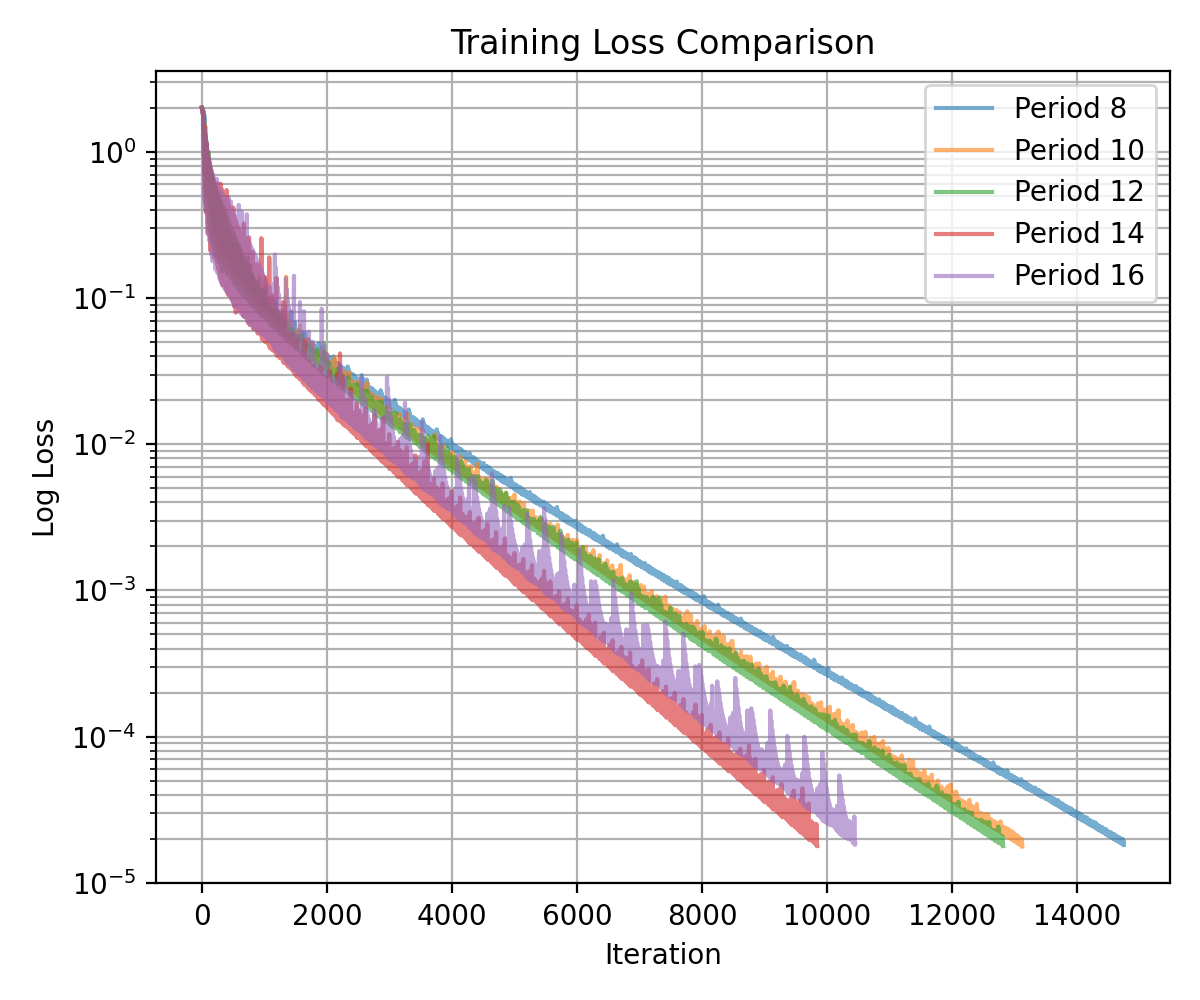}
    \caption{Training Loss}
    \label{fig:loss}
  \end{subfigure}%
  \hfill
  \begin{subfigure}[t]{0.24\textwidth}
    \centering
    \includegraphics[width=\textwidth]{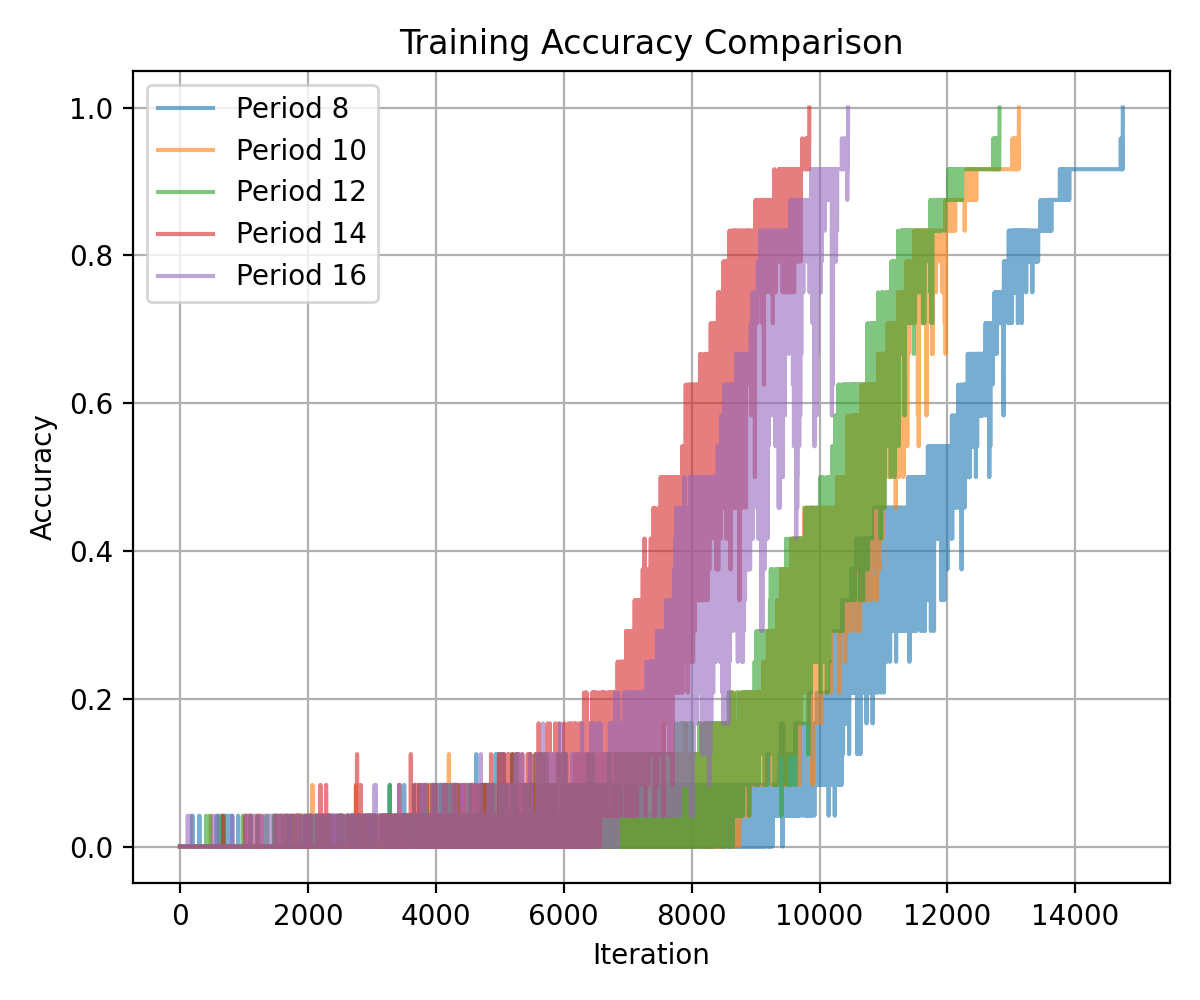}
    \caption{Training Accuracy}
    \label{fig:accuracy}
  \end{subfigure}
  \hfill
  \begin{subfigure}[t]{0.24\textwidth}
    \centering
    \includegraphics[width=\textwidth]{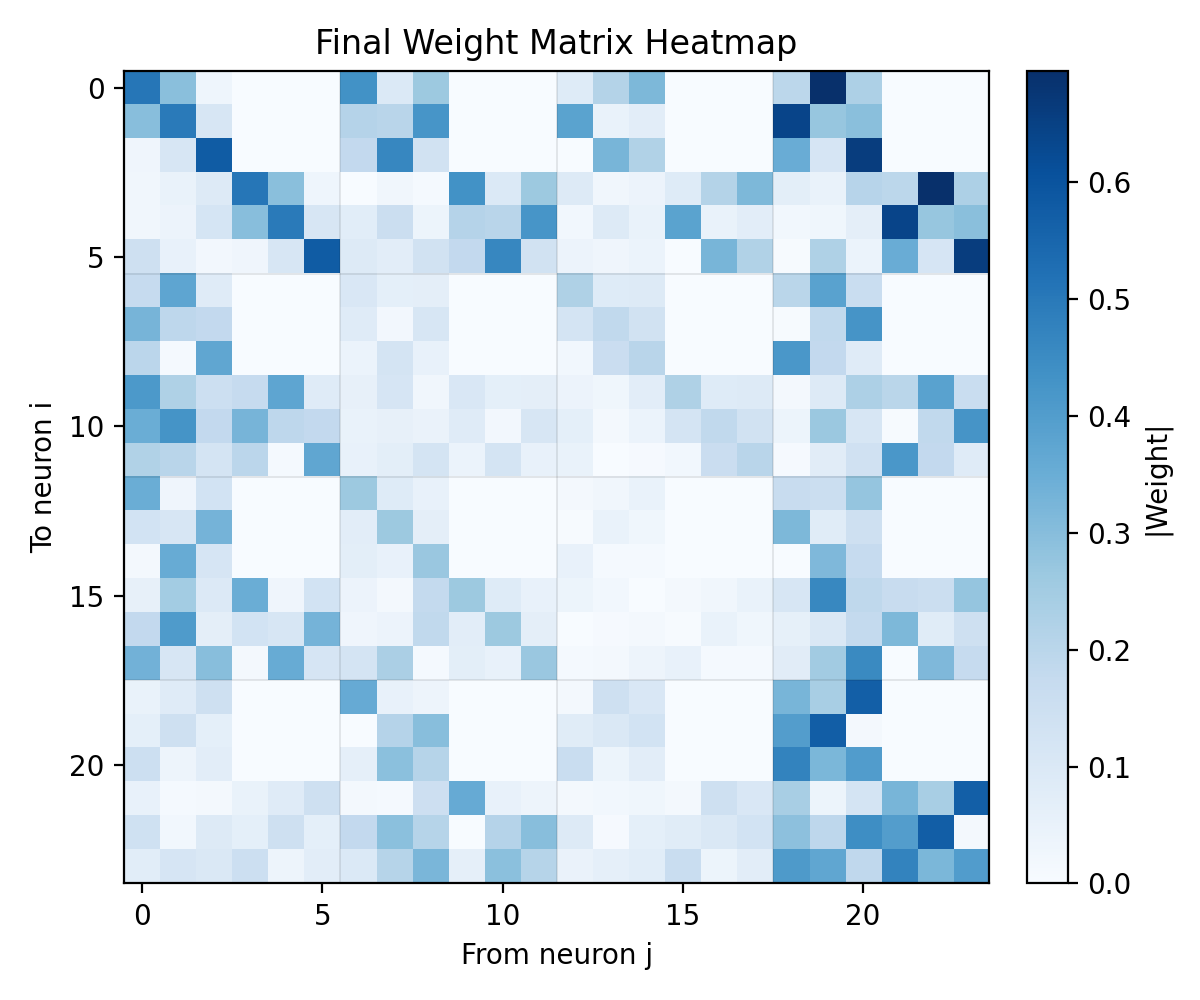}
    \caption{Lie Group Weights}
    \label{fig:heatmap1}
  \end{subfigure}
  \hfill
  \begin{subfigure}[t]{0.24\textwidth}
    \centering
    \includegraphics[width=\textwidth]{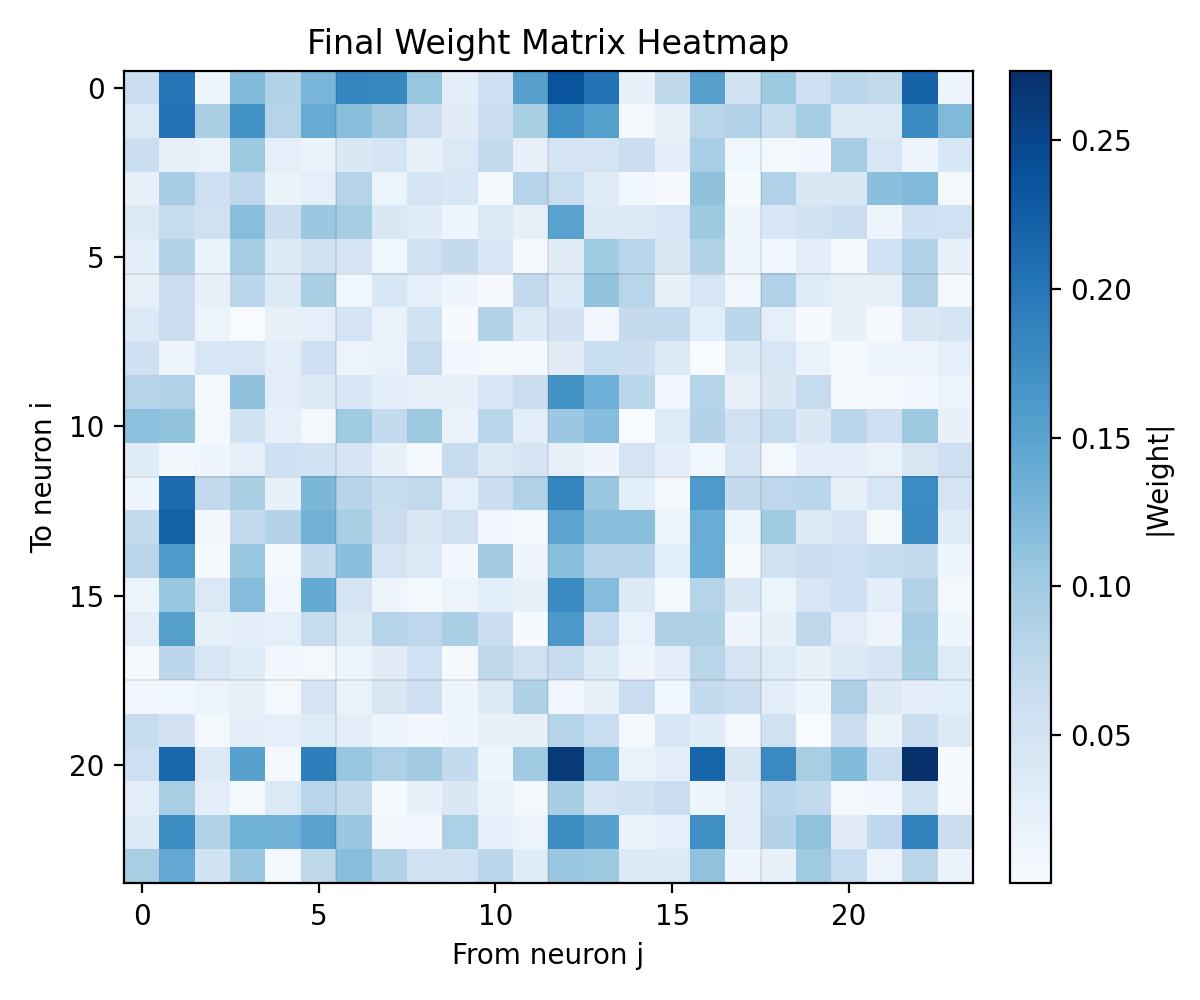}
    \caption{Euclidean Weights}
    \label{fig:heatmap2}
  \end{subfigure}
\label{fig:experiment2}
\caption{(a). Loss curves of projective manifold learning. Periodic projection will cause fluctuation in loss descent. As shown in the plot, larger period $\mathscr{P}\in\mathbb{N^*}$, which means more iterations of steepest gradient step, converges faster, yet the convergences are noticeable perturbative. The range of projection period is $8\le\mathscr{P}\le16$. Out of this interval, too small periods take too much time to train, and too large periods fail the training as the weights seriously deviate from Lie group manifold. (b). With an appropriate range of periods, learning process reaches perfect accuracy where planning trajectory evolves to the target posture precisely. (c). Trained weights not only guarantee the correct output, i.e. the network equilibrium, but also preserve the geometric structure. As illustrated, every $6\times 6$ block of weight matrix is constrained in $\mathrm{SE}(3)$ induced adjoint action, i.e. $\mathcal{L}[\Ad_{\mathrm{SE}(3)}]$, with symmetric blocks of Eq. (\ref{eq:L[Ad]}). The interactions of neurons are therefore equipped with geometric coupling of $\mathrm{SE}(3)$. (d). Ablation of manifold projection. weight matrix loses the block-wise structure. Although target equilibrium is still reached, neural network dynamics loses its geometric and physical meaning, thereby no longer provides geometrically interpretable recurrent coupling.}
\end{figure}

\begin{figure}[h]
    \centering
    %(a)
    \begin{subfigure}{0.48\textwidth}
        \centering
        \includegraphics[width=\linewidth]{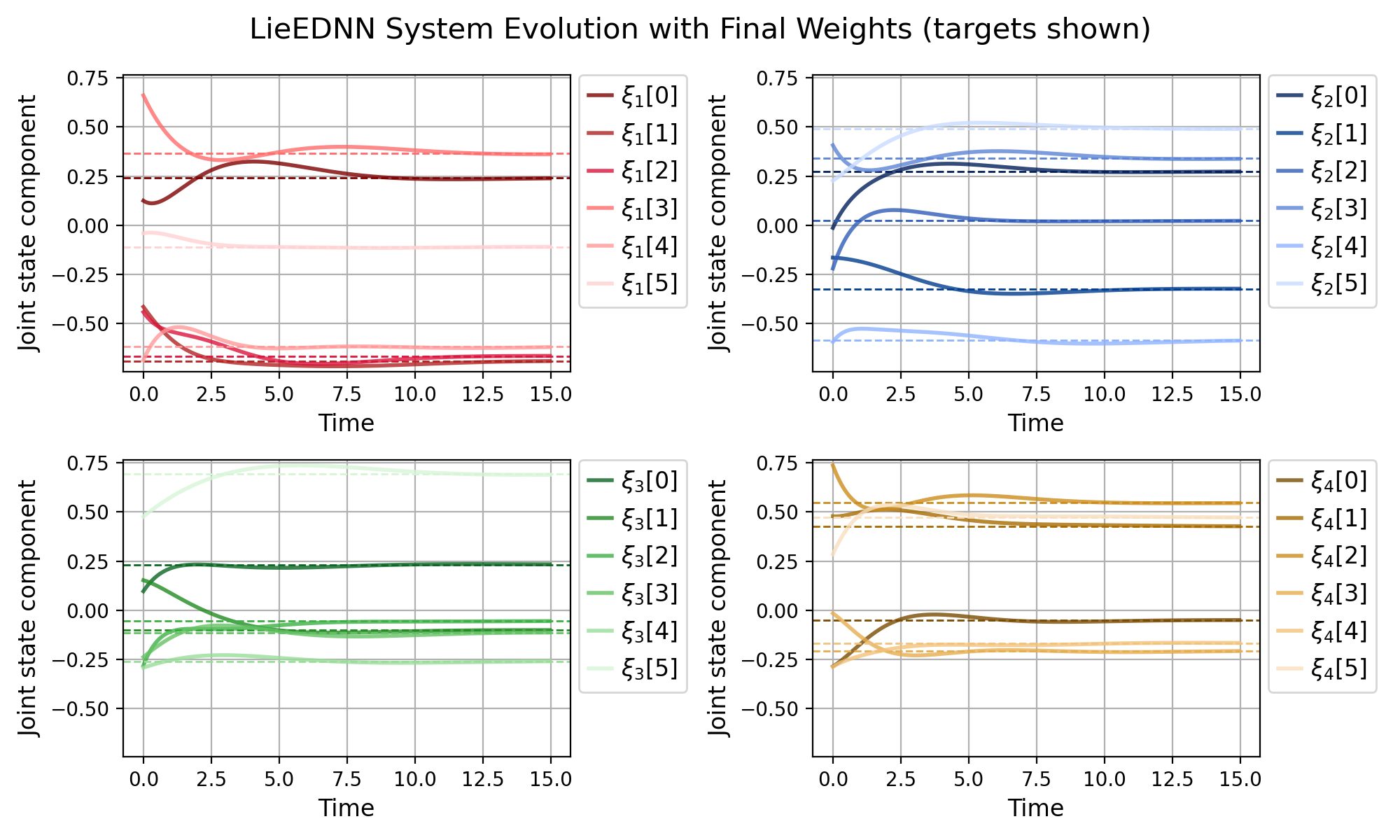}
        \caption{Network Dynamics Trajectory}
        \label{fig:trajectory}
    \end{subfigure}
    \hspace*{0.2cm}
    %(b)
    \begin{subfigure}{0.48\textwidth}
        \centering
        \includegraphics[width=\linewidth]{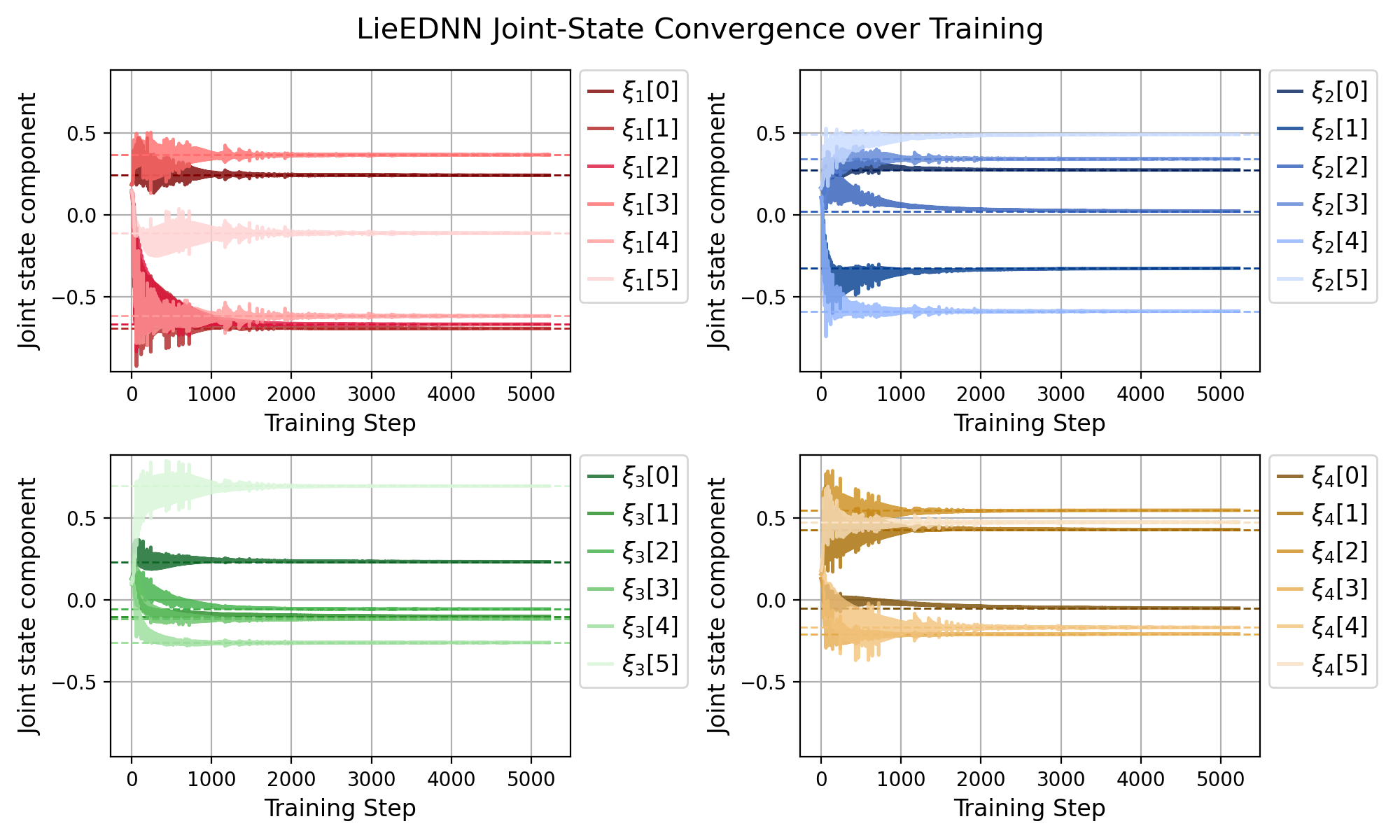}
        \caption{Changes of Equilibrium over Training}
        \label{fig:training}
    \end{subfigure}
    \caption{(a). Evolution trajectory of trained LieEDNN. Network dynamics converge to the target equilibrium $\boldsymbol{\xi}_d$, where $\boldsymbol{\xi}_i$ denote the $i$th representative neuron parameterized in $\mathbb{R}^6$, i.e. $\mathcal{P}(\boldsymbol{\xi}_i),\ \boldsymbol{\xi}_i\in\mathfrak{se}(3)$. By retraction function $\mathrm{Exp}(\boldsymbol{\xi}_i)$, the neural state is interpreted in $\mathrm{SE}(3)$, which corresponding to the posture of a telescopic joint, as in Fig. \ref{fig:arm}. LieEDNN generates a vector field in the parameter space of the manipulator that guarantees the convergence and stability by Thm. \ref{Thm:EandU}. (b). By taking the gradient of the target equilibrium about weights in Sec. \ref{sec:learning}, the equilibrium of the network is changing over training epochs. The dashed lines represent the targets, the same as the trained network converges to in Fig. \ref{fig:trajectory}. Experiments shows that the equilibria converge to the target as designed.}
    \label{fig:experiment1}
\end{figure}

\section{Discussion}
\textbf{Application Perspective}. For manipulator planning, LieEDNN can be positioned as a geometrically structured neural dynamical framework, as shown in Sec. \ref{sec:experiment}. Classical robotic motion planning is commonly studied using sampling-based planners such as RRT-Connect and RRT$^*$ \cite{rrtconnect,rrtstar}, and trajectory optimization approaches such as CHOMP, STOMP, and TrajOpt \cite{chomp,stomp,trajopt}, which search in the configuration space and optimize feasible trajectories under constraints of  collision, joint, velocity, and acceleration, etc. In contrast, LieEDNN learns a continuous autonomous neural dynamical system whose desired manipulator state is represented as an equilibrium, while the transient motion is generated by the learned dynamics itself. This places it to the category of dynamical system based motion generation, such as dynamic movement primitives (DMP) and stable learned dynamical systems (stableDS) \cite{stableDS}, Riemann motion policies (RMP) \cite{rmpflow} and quaternionic Hopfield structured neural planner (QSHNN) \cite{qshnn}. The main distinction is that the recurrent neural coupling is constrained by Lie group induced adjoint representations, so the manipulator state is not treated as an unconstrained Euclidean vector but an algebraic combination. And the framework can be extended to rich representations more than $\mathrm{SE}(3)$. For mixed revolute-prismatic or telescopic manipulators, the \(\mathrm{SE}(3)\) structured blocks provide a principled way to encode rotational and translational coupling, as described in Sec. \ref{sec:methodology}. Thus, LieEDNN provides a geometrically consistent neural motion generator that can be further combined with trajectory constraints such as joint limits and collision avoidance.

\textbf{Comparative Analysis}. Compared with standard feedforward neural networks, LieEDNN is not a static input-output approximator, but a continuous-time recurrent neural dynamical system whose output is represented by the equilibrium of the dynamics. Compared with LSTM \cite{LSTM} and Transformer architectures \cite{transformer}, LieEDNN is not designed primarily for sequence modelling through hidden-state memory or attention. Instead, it focuses on learning stable autonomous dynamics with explicit geometric constraints on recurrent coupling. In this sense, LieEDNN is closer to continuous Hopfield-type neural networks \cite{Hopfield:2007} and neural ODEs \cite{neuralODE}, but it differs from them by imposing Lie group induced block structures on the weight matrix and strict stability guarantees under the sufficient conditions in Thm.~\ref{Thm:EandU} enforced in our experiments. Compared with complex-valued \cite{complex-valued,complex-valued2} and quaternion-valued neural networks \cite{quaternionRecurrent,quaternionConvolution}, the proposed framework is more general: the algebraic structure is not fixed to a specific number field, but can be induced from any chosen matrix Lie group including complex number by Lie group $\mathbb{C}^\times\!\cong\mathbb{R}^+\!\times \!\mathrm{U}(1)\!\cong\!\mathbb{R}\!\times\!\mathrm{SO}(2)$ and quaternion by Lie group $\mathbb{H}^\times\!\cong\mathbb{R}^+\!\times \!\mathrm{S}^3\!\cong\!\mathbb{R}\!\times\!\mathrm{SU}(2)$. Moreover, its main value lies in reliability and interpretability: stability can be rigorously studied through dynamical systems theory, while geometric consistency is preserved by projective manifold learning. This makes LieEDNN suitable for tasks of planning dynamical systems equipped with clear physical and geometric structures rather than only optimizing a prediction driven by large datasets.
% Limitations and Future Work for LieEDNN can be found in App. \ref{J}.

\section{Conclusions}
In this paper, we proposed LieEDNN, a Lie group embedded neural dynamical framework for learning stable continuous-time dynamics under geometric constraints. By transferring neural states from a Lie group \(G\) to its Lie algebra \(\mathfrak{g}\), and representing the adjoint action \(\Ad_G(\cdot)\) through the induced linear mapping \(\mathcal{L}[\Ad_G]\), LieEDNN remains compatible with standard neural network operations while preserving the underlying geometric structure. We also establish asymptotic stability of the equilibrium and derive supervised learning rules for both structured weights and connection strengths through periodic projective manifold learning. Experiments on telescopic manipulator planning problem with Lie group \(\mathrm{SE}(3)\) show that LieEDNN can learn and generate geometrically consistent trajectories in a high-dimensional parameter space. These results suggest a principled route toward neural models for learning structured dynamical systems effectively with geometric interpretability.

\small{
\bibliographystyle{plainnat}
\bibliography{bibliography}
}
%%%%%%%%%%%%%%%%%%%%%%%%%%%%%%%%%%%%%%%%%%%%%%%%%%%%%%%%%%%%
\newpage
\appendix

\section{Proof of Lem. \ref{lem:metric projection}: Metric Projection to Manifold $\mathrm{SO}(3)$}\label{A}
\begin{lemma}[Metric Projection to \(\mathrm{SO}(3)\)]
Given an arbitrary matrix \(\Pi\in\mathbb{R}^{3\times 3}\), the metric projection from \(\Pi\) to \(\mathrm{SO}(3)\) under the Frobenius norm is:
\begin{equation}
\mathrm{Proj}_{\mathrm{SO}(3)}(\Pi)=U\,\mathrm{diag}\bigl(1,1,\det(UV^{\mathrm T})\bigr)V^{\mathrm T}
\end{equation}
where \(\Pi=U\Sigma V^{\mathrm T}\) is the singular value decomposition (SVD) of \(\Pi\), with \(\Sigma=\mathrm{diag}(\sigma_1,\sigma_2,\sigma_3)\) and singular values \(\sigma_1\ge\sigma_2\ge\sigma_3\ge0\). Equivalently, $\mathrm{Proj}_{\mathrm{SO}(3)}(\Pi)=\operatorname*{arg\,min}_{R\in\mathrm{SO}(3)}\|\Pi-R\|_F$.
\end{lemma}
\textit{\textbf{Proof:}}
To find the metric projection of $\Pi$ under Frobenius norm is equivalent to solve the optimization problem: $\operatorname*{arg\,min}_{R\in\mathrm{SO}(3)}\|\Pi-R\|_F$. Using the fact that \(\ \|A\|_F^2=\operatorname{tr}(A^{\mathrm T}A)\) and \(R^{\mathrm T}R=I,\ R\in\mathrm{SO}(3)\), we have:
\begin{align}
\|\Pi-R\|_F^2
&=\operatorname{tr}\!\bigl((\Pi-R)^{\mathrm T}(\Pi-R)\bigr)\\
&=\operatorname{tr}(\Pi^{\mathrm T}\Pi)-\operatorname{tr}(\Pi^{\mathrm T}R)-\operatorname{tr}(R^{\mathrm T}\Pi)+\operatorname{tr}(R^{\mathrm T}R)\\
&=\operatorname{tr}(\Pi^{\mathrm T}\Pi)-2\operatorname{tr}(\Pi^{\mathrm T}R)+3
\end{align}
Since $\operatorname{tr}(\Pi^{\mathrm T}\Pi)$ is a constant, minimizing \(\|\Pi-R\|_F^2\) is equivalent to maximizing \(\operatorname{tr}(\Pi^{\mathrm T}R)\). Let the SVD of $\Pi$ be $
\Pi=U\Sigma V^{\mathrm T}$, and we define $Q:=U^{\mathrm T}RV$.
Then \(Q\in O(3)\), \(\det(Q)=\det(UV^{\mathrm T})=\pm1\). By the cyclic property of trace, the objective quantity now becomes:
\begin{equation}
\operatorname{tr}(\Pi^{\mathrm T}R)=\operatorname{tr}(V\Sigma U^{\mathrm T}R)=\operatorname{tr}(\Sigma U^{\mathrm T}RV)=\operatorname{tr}(\Sigma Q)
\end{equation}
Thus the problem reduces to maximizing:
\begin{equation}
\operatorname{tr}(\Sigma Q)=\sigma_1Q_{11}+\sigma_2Q_{22}+\sigma_3Q_{33}
\end{equation}
If \(\det(Q)=1\), then \(\operatorname{tr}(\Sigma Q)\le \sigma_1+\sigma_2+\sigma_3\), with equality established if and only if \(Q=\mathbb{I}_3\). Hence \(R=UV^{\mathrm T}\). If \(\det(Q)=-1\), then \(Q\) has eigenvalues \(e^{\pm i\theta}\) and \(-1\), so \(\operatorname{tr}(Q)=-1+2\mathrm{cos}\theta\le1\). Therefore,
\begin{align}
\operatorname{tr}(\Sigma Q)&=(\sigma_1-\sigma_3)Q_{11}+(\sigma_2-\sigma_3)Q_{22}+\sigma_3(Q_{11}+Q_{22}+Q_{33})\\
&=(\sigma_1-\sigma_3)Q_{11}+(\sigma_2-\sigma_3)Q_{22}+\sigma_3\operatorname{tr}(Q)\\
&\le \sigma_1+\sigma_2-\sigma_3,
\end{align}
and equality holds if and only if \(Q=\mathrm{diag}(1,1,-1)\). Hence, $R=U\,\mathrm{diag}(1,1,-1)\,V^{\mathrm T}$.\\
Combining both two cases, we have the metric projection of $\Pi$ to be:
\begin{equation}
\mathrm{Proj}_{\mathrm{SO}(3)}(\Pi)=U\,\mathrm{diag}\bigl(1,1,\det(UV^{\mathrm T})\bigr)V^{\mathrm T}
\end{equation}
which is the conclusion we need for Lem. \ref{lem:metric projection}. \hfill$\square$

\section{Proof of Thm. \ref{thm:Induced Linear Mapping as a Lie Group Manifold}: Induced Linear Mapping as a Lie Group Manifold}\label{B}
\begin{theorem}[Induced Linear Mapping as a Lie Group Manifold]$\mathcal{L}{[}Ad_{\mathrm{SE}(3)}{]}$ is topologically closed in $\mathrm{GL}(6,\mathbb{R})$, which forms a Lie group and an underlying manifold defined with matrix representation in Def. \ref{def:Induced Linear Mapping by Adjoint Action}. Further, the group operation is compatible with matrix multiplication.
\end{theorem}
\textit{\textbf{Proof:}}
For arbitrary two elements $H_1,H_2$ in $\mathscr{L}:=\mathcal{L}{[}Ad_{\mathrm{SE}(3)}{]}$, and we calculate their product:
\begin{equation}
H_{1,2}:=
\begin{bmatrix}
R_{1,2}&O_{3\times 3}\\
[t_{1,2}]^\wedge R_{1,2}&R_{1,2}
\end{bmatrix},\quad
H_1\circ H_2=
\begin{bmatrix}
R_1R_2&O_{3\times 3}\\
[t_1]^\wedge R_1R_2+R_1[t_2]^\wedge R_2&R_1R_2
\end{bmatrix}
\end{equation}
Using \([Ru]^\wedge=R[u]^\wedge R^T\), it follows that $R_1[t_2]^\wedge R_2=[R_1t_2]^\wedge R_1R_2$. Hence,
\begin{equation}
H_1\circ H_2=
\begin{bmatrix}
R_1R_2&O_{3\times 3}\\
[t_1+R_1t_2]^\wedge R_1R_2&R_1R_2
\end{bmatrix}
\in \mathscr{L}.
\end{equation}
which means the group operation is closed. And the inverse of $H_1$ satisfying $H_1\circ H_1^{-1}=\mathbb{I}$ exists:
\begin{equation}
H_1^{-1}=\begin{bmatrix}
R_1^{-1}&O_{3\times 3}\\
[-R_1^{-1}t_1]^\wedge R_1^{-1}&R_1^{-1}
\end{bmatrix}\in\mathscr{L}
\end{equation}
Thus \(\mathcal{L}[Ad_{\mathrm{SE}(3)}]\) is a subgroup of \(\mathrm{GL}(6,\mathbb{R})\). Now let a sequence of $H_n$ in $\mathscr{L}$ converges: 
\begin{equation}H_n=
\begin{bmatrix}
R_n&O_{3\times 3}\\
[t_n]^\wedge R_n&R_n
\end{bmatrix}\rightarrow
H_\infty:=
\begin{bmatrix}
\mathcal{A}&O_{3\times 3}\\
\mathcal{C}&\mathcal{A}
\end{bmatrix}
\end{equation}
Since \(R_n\to A\) and \(\mathrm{SO}(3)\) is closed, we have \(A\in \mathrm{SO}(3)\). Moreover, $[t_n]^\wedge=([t_n]^\wedge R_n)R_n^T\to CA^T$. As the skew-symmetric matrices form a closed set, \(CA^T=[t]^\wedge\) for a \(t\in\mathbb{R}^3\). Hence $
C=[t]^\wedge A$. So, $H_\infty\in\mathscr{L}$. Therefore \(\mathcal{L}[Ad_{\mathrm{SE}(3)}]\) is topologically closed in \(\mathrm{GL}(6,\mathbb{R})\). By Thm. \ref{thm:closed subgroup theorem}, the operation closure and topological closure is established, thus it is a matrix Lie group.\hfill$\square$

\section{Proof of Thm. \ref{thm:Linear Representation of Adjoing Action}: Linear Representation of Adjoing Action}\label{C}
\begin{theorem}[Linear Representation of Adjoint Action]For Lie algebra $\mathfrak{se}(3)$ defined in Def. \ref{def:se(3) and SE(3)}, each element is parameterized by $\boldsymbol{\omega}$ and $\boldsymbol{p}$. Denote the vector parameterization of $X\in\mathfrak{se}(3)$ by $P(X):=[\boldsymbol{\boldsymbol{\omega}}\ \boldsymbol{p} ]^{\mathrm{T}}\in\mathbb{R}^6$, then there exists a linear mapping $\mathcal{L}[{\mathcal{A}d_G}]\in\mathcal{L}(\mathbb{R}^6)$ induced by the adjoint action of $\mathrm{SE}(3)$ on $\mathfrak{se}(3)$ of the form in Def. \ref{def:Induced Linear Mapping by Adjoint Action}. Further, we have
\begin{equation}
    P(\mathcal{A}d_g(X))=\mathcal{L}[\mathcal{A}d_g]\circ P(X)
\end{equation}
i.e. the linear operator on parameterized $\mathfrak{se}(3)$ vector is equivalent to parameterizing the adjoint action $\mathcal{A}d_g(\cdot)$ of a fixed $g\in \mathrm{SE}(3)$ on $X\in\mathfrak{se}(3)$.
\label{thm:}
\end{theorem}
\textit{\textbf{Proof:}} Take $
g=
\begin{bmatrix}
R&t\\
0&1
\end{bmatrix}\in \mathrm{SE}(3),\
X=
\begin{bmatrix}
[\omega]^\wedge&p\\
0&0
\end{bmatrix}\in \mathfrak{se}(3).$
Then adjoint action is: \begin{equation}\mathcal{A}d_g(X)=gXg^{-1}=\begin{bmatrix}
R[\omega]^\wedge R^T&R p-R[\omega]^\wedge R^T t\\
0&0
\end{bmatrix}
\end{equation}
by Def. \ref{def:adjoint action}. Using \(R[\omega]^\wedge R^T=[R\omega]^\wedge\) and $-R[\omega]^\wedge R^T t=[t]^\wedge R\omega$, we obtain: 
\begin{equation}
\mathcal{A}d_g(X)=
\begin{bmatrix}
[R\omega]^\wedge&[t]^\wedge R\omega+Rp\\
0&0
\end{bmatrix}
\end{equation} 
Therefore, under \(P(X)=[\omega\ \ p]^T\), we can calculate:
\begin{equation}
P(\mathcal{A}d_g(X))
=
\begin{bmatrix}
R\omega\\
[t]^\wedge R\omega+Rp
\end{bmatrix}
=
\begin{bmatrix}
R&0\\
[t]^\wedge R&R
\end{bmatrix}
P(X)=\mathcal{L}[\mathcal{A}d_g]P(X)
\label{eq:3.9}
\end{equation}
which is just the formula in need as Eq. (\ref{eq: Linear Representation of Adjoint Action}) of the theorem. \hfill$\square$

\section{Linear Representation of Adjoint Action for General Lie Group}\label{D}
Let $G \leq \GL(n,\R)$ be a closed matrix Lie group, and write $\g := T_e G$ for its Lie algebra, a linear subspace of $\gl(n,\R) = \R^{n\times n}$ of dimension $d$. Recall that in Def. \ref{def:adjoint action}, for a fixed $g \in G$, the \emph{adjoint action} of $g$ on $\g$ is then $\Ad_g : \g \rightarrow \g, \ \Ad_g(X) = g X g^{-1}$. Fix once and for all a basis $\mathcal{B} = \{E_1, \dots, E_d\}$ of $\g$ as a vector space, and denote the parameterization mapping as in Thm. \ref{eq: Linear Representation of Adjoint Action} with basis expansion of $\boldsymbol{\xi}\in\mathfrak{g}$:
\begin{equation}
\label{eq:param}
\mathcal{P} : \g \xrightarrow{\,\sim\,} \R^d, \quad \mathcal{P}(\boldsymbol{\xi})=\mathcal{P}(\xi^i E_i ) = (\xi^1, \dots, \xi^d)^{\mathrm T}
\end{equation}
which is a linear isomorphism, since $\mathfrak{g}$ is essentially a vector space of dimension $d$. We have the following theorem to describe the property of adjoint action induced linear mapping:

\begin{theorem}[General Linear Representation of Adjoint Action]
\label{thm:General Linear Representation of Adjoint Action}
For each $g \in G$, there exists a unique linear automorphism $\Lmat[\Ad_g] \in \GL(d,\R)$ satisfying:
\begin{equation}
\label{eq:intertwining}
\Param\bigl(\Ad_g(X)\bigr) = \Lmat[\Ad_g]\, \Param(X), \quad \forall X \in \g.
\end{equation}
Explicitly, the linear representation is defined by:
\begin{equation}
\Lmat[\Ad_g] = \Param \circ \Ad_g \circ \Param^{-1}
\label{eq:LinearRepresentation}
\end{equation}
Furthermore, the map $\Lmat : G \to \GL(d,\R)$ with $g \mapsto \Lmat[\Ad_g]$ is a smooth homomorphism of Lie groups, which associate the group elements to the adjoint action mapping respectively.
\end{theorem}

{\it\textbf{Proof:}}
{Existence.} By Thm .\ref{thm:LieAlgebraAutomorphism}, $\Ad_g \in \mathrm{Aut}(\g)$. Since $\Param$ is a linear isomorphism, the composition $P\circ\Ad_g\circ\mathcal{P}^{-1}$ lies in $\GL(d,\R)$. $\forall X \in \g$, $\Lmat[\Ad_g]\,\Param(X) = \Param\circ\Ad_g \circ\Param^{-1}\circ\Param(X) = \Param\bigl[\Ad_g(X)\bigr]$, which establishes \eqref{eq:intertwining}.

{Uniqueness.} Let $M \in \GL(d,\R)$ be any linear automorphism satisfying \eqref{eq:intertwining} in place of $\Lmat[\Ad_g]$. Subtracting the two identities yields $\bigl(M - \Lmat[\Ad_g]\bigr)\,\Param(X) \equiv 0 \quad \forall X \in \g$. Because $\Param$ is surjective onto $\R^d$, i.e. $\mathrm{range}\mathcal{P}=\mathbb{R}^d$, this forces $M - \Lmat[\Ad_g] = 0$, hence $M = \Lmat[\Ad_g]$. The existence of linear representation is unique.

{Homomorphism property.} For $g_1, g_2 \in G$, the identity $\Ad_{g_1 * g_2} = \Ad_{g_1} \circ \Ad_{g_2}$ combined with Eq. \ref{eq:LinearRepresentation} gives
\begin{align}
\Lmat[\Ad_{g_1 *g_2}]&= \Param \circ \Ad_{g_1} \circ \Ad_{g_2} \circ \Param^{-1}\\
&= \bigl(\Param \circ \Ad_{g_1} \circ \Param^{-1}\bigr)\circ\bigl(\Param \circ \Ad_{g_2} \circ \Param^{-1}\bigr)
= \Lmat[\Ad_{g_1}]\, \Lmat[\Ad_{g_2}].
\label{eq:homomorphismL}
\end{align}
Thus, the linear representation of adjoint action is a Lie group homomorphism. And the Smoothness of $\Lmat$ follows from that of $g \mapsto \Ad_g$ and the endomorphism on $\mathfrak{g}$ by $\mathcal{P}$ and $\mathcal{P}^{-1}$ as appearing in Eq. \ref{eq:LinearRepresentation}.\hfill\(\square\)

\section{Properties of Induced Linear Mapping for General Lie groups}\label{E}
For a general Lie group $G$ of dimension $d$, the induced linear mapping of adjoint action is defined by:
\begin{equation}\Lmat[\mathcal{A}d_G]:=\bigl\{\left.\Lmat[\Ad_g]\ \right| g \in G\bigr\} \subseteq \mathbb{R}^{d\times d}
\end{equation} 
Firstly, we prove \(\Lind[\Ad_G]\) is a subgroup of \(\GL(d,\R)\). For any \(g_1,g_2\in G\), the adjoint representation satisfies:
\begin{equation}
\Ad_{g_1} \Ad_{g_2}=g_1 g_2 (\cdot) g_2^{-1}g_1^{-1}=(g_1 g_2) (\cdot)(g_1g_2)^{-1}=\Ad_{g_1g_2}
\end{equation}
After choosing a basis of Lie algebra \(\mathfrak g\) represented in $\mathbb{R}^d$, by Eq. (\ref{eq:homomorphismL}) of Thm. \ref{thm:General Linear Representation of Adjoint Action} this gives:
\begin{equation}
\Lmat[\Ad_{g_1}]\Lmat[\Ad_{g_2}]
=
\Lmat[\Ad_{g_1g_2}]
\in \mathscr{L}=\Lmat[\Ad_G].
\end{equation}
Moreover, $\Lmat[\Ad_e]=I_d$, where \(e\in G\) is the identity element, and $\Lmat[\Ad_g]^{-1}=\Lmat[\Ad_{g^{-1}}]$. Hence \(\Lind[\Ad_G]\) is closed under multiplication and inverse, and therefore forms a subgroup of \(\GL(d,\R)\).

Next, define the kernel of the induced linear representation by:
\begin{equation}
\ker\!\Lmat
:=
\{g\in G\mid \Lmat[\Ad_g]=\mathbb{I}_d\}=\left\{g\in G\mid \operatorname{Ad}_g(X)=X,\ \forall X\in\mathfrak g\right\}.
\label{eq:kernalL}
\end{equation}
Since \(\Lmat:G\to\GL(d,\R)\) is a Lie group homomorphism by Thm. \ref{thm:General Linear Representation of Adjoint Action}, \(\ker\!\Lmat\) is a closed normal Lie subgroup of \(G\). Therefore, the quotient \(G/\ker\Lmat\) is a quotient Lie group. Moreover, by the first isomorphism theorem for groups \cite{groupsandsymmetry}, the map
\begin{equation}
\overline{\Lmat}:G/\ker\Lmat\to \Lmat[\Ad_G],
\quad
\overline{\Lmat}(g\ker\Lmat)=\Lmat[\Ad_g],
\end{equation}
is a well-defined group isomorphism. Indeed, if \(g_1\ker\Lmat=g_2\ker\Lmat\), then \(g_2^{-1}g_1\in\ker\Lmat\), and hence
\begin{equation}
\Lmat[\Ad_{g_1}]
=
\Lmat[\Ad_{g_2}]\Lmat[\Ad_{g_2^{-1}g_1}]
=
\Lmat[\Ad_{g_2}].
\end{equation}
Thus \(\overline{\Lmat}\) is well-defined. It is clearly surjective by the definition of \(\Lmat[\Ad_G]\). If $\overline{\Lmat}(g_1\ker\Lmat)=\overline{\Lmat}(g_2\ker\Lmat)$ then $\Lmat[\Ad_{g_2^{-1}g_1}]=\mathbb{I}_d$
so \(g_2^{-1}g_1\in\ker\Lmat\), and therefore \(g_1\ker\Lmat=g_2\ker\Lmat\). Hence \(\overline{\Lmat}\) is injective. Consequently,
\begin{equation}
\Lmat[\Ad_G]\cong G/\ker\Lmat
\end{equation}
as a quotient Lie group. Therefore, for a general Lie group \(G\), the induced adjoint matrix set \(\Lmat[\Ad_G]\) is naturally a Lie group through the quotient structure \(G/\ker\Lmat\). If, in addition, \(\Lmat[\Ad_G]\) is topologically closed in \(\GL(d,\R)\), then it is also an embedded matrix Lie subgroup of \(\GL(d,\R)\) by the closed subgroup theorem in Thm. \ref{thm:closed subgroup theorem}.

The kernel consists of the group elements acting trivially on the Lie algebra, as shown in Eq. (\ref{eq:kernalL}). For a connected matrix Lie group, these are precisely the central elements invisible to the adjoint action. Hence \(\mathcal{L}[\operatorname{Ad}_G]\) naturally identifies elements differing by \(\ker\mathcal{L}\), and carries the quotient Lie group structure \(G/\ker\mathcal{L}\). For example, for \(G=\mathrm{SU}(2)\), the center \(\{\pm I\}\) lies in the kernel and \(\mathrm{SU}(2)/\{\pm I\}\cong \mathrm{SO}(3)\). More generally, compact groups such as \(\mathrm{SU}(n)\), \(\mathrm{Sp}(n)\), and \(\mathrm{SO}(n)\) \cite{LieTheory} have closed adjoint images by Thm.~\ref{thm:compact-adjoint-image}, so \(\mathcal{L}[\operatorname{Ad}_G]\) is also an embedded matrix Lie group. However, this is not a necessary condition. In particular $\mathrm{SE}(3)$ is not compact but still equipped with adjoint induced mapping as matrix Lie group as shown in Thm. \ref{thm:Induced Linear Mapping as a Lie Group Manifold}. 
\begin{theorem}[Compactness as Sufficient Condition]
\label{thm:compact-adjoint-image}
If \(G\) is compact, then \(\Lmat[\Ad_G]\) is a compact, hence closed, subgroup of \(\GL(d,\R)\). Consequently, \(\Lmat[\Ad_G]\) is an embedded matrix Lie subgroup of \(\GL(d,\R)\).
\end{theorem}
{\it\textbf{Proof:}}
By Thm.~\ref{thm:General Linear Representation of Adjoint Action}, the map $\Lmat:G\to \GL(d,\R),\ g\mapsto \Lmat[\Ad_g]$ is smooth, hence continuous. Since \(G\) is compact, its image $\Lmat[\Ad_G]=\Lmat(G)$ is compact in \(\GL(d,\R)\). Since \(\GL(d,\R)\) is Hausdorff, every compact subset is closed. Therefore \(\Lmat[\Ad_G]\) is a topologically closed subgroup of \(\GL(d,\R)\). By the closed subgroup theorem in Thm.~\ref{thm:closed subgroup theorem}, \(\Lmat[\Ad_G]\) is an embedded matrix Lie subgroup of \(\GL(d,\R)\). \hfill\(\square\)

\section{Metric Projection to Manifold \(\mathscr{L}:=\mathcal{L}{[}Ad_{\mathrm{SE}(3)}{]}\) Block-wisely}\label{F}
\begin{theorem}[Metric Projection to Manifold $\mathscr{L}:=\mathcal{L}{[}Ad_{\mathrm{SE}(3)}{]}$ Block-wisely] Given an arbitrary matrix $\Theta\in\mathbb{R}^{6\times 6}$ with $3\times 3$ block division $\mathcal{A},\mathcal{B},\mathcal{C},\mathcal{D}\in\mathbb{R}^{3\times 3}$, then the block-wise projection from $\Theta$ to $\mathscr{L}$ under the metric of Frobenius norm is given by the following formula:
\begin{align}
    \mathrm{Proj}_\mathscr{L}(\Theta)=\left[\begin{matrix}
    \mathrm{Proj}_{\mathrm{SO}(3)}(\mathcal{M})&O_{3\times 3}\\[0.2cm]
    [\boldsymbol{t_\Theta}]^\wedge\mathrm{Proj}_{\mathrm{SO}(3)}(\mathcal{M})&\mathrm{Proj}_{\mathrm{SO}(3)}(\mathcal{M})
    \end{matrix}\right]\in\mathscr{L}
    \label{eq:metric projection App}
\end{align}
where we denote the arithmetic average of $\mathcal{A},\ \mathcal{B}$ and the skew-symmetry part of $\mathcal{C}R^\mathrm{T}$ by:
\begin{equation}
    \Theta:=\left[\begin{matrix}
    \mathcal{A}&\mathcal{B}\\[0.1cm]
    \mathcal{C}&\mathcal{D}
    \end{matrix}\right],\quad \mathcal{M}:=\frac{\mathcal{A}+\mathcal{D}}{2},\quad [\boldsymbol{t}_\Theta]^\wedge:=\frac{\mathcal{C}R^\mathrm{T}-R\ \!\mathcal{C^\mathrm{T}}}{2},\quad R:=\mathrm{Proj}_{\mathrm{SO}(3)}(\mathcal{M}).
    \label{eq:metric projection2 App}
\end{equation}
\label{thm:Metric Projection to Manifold L App}
\end{theorem}
{\it\textbf{Proof:}} Note that the definition of projection to $\mathcal{L}{[}Ad_{\mathrm{SE}(3)}{]}$ is not straight and strict as the metric projection to $\mathrm{SO}(3)$ in Lem. \ref{lem:metric projection}, but a block-wise projection. This is based on the block structure of $\mathcal{L}{[}Ad_{\mathrm{SE}(3)}{]}$ as seen in Def. \ref{def:L}, where for each structure block, we do the projection of from block at corresponding position and combine a whole manifold projection $\mathrm{Proj}_\mathscr{L}(\Theta)$. This definition leads to the following theorem for projective manifold learning to preserve the geometric consistency of LieEDNN.

The block-wise projection given in Def. \ref{def:L} is a well-defined block-wise projection mapping to the manifold $\mathcal{L}{[}Ad_{\mathrm{SE}(3)}{]}$, and the constraint on the manifold preserves the consistency of the adjoint Lie group action $\mathcal{A}d_{\mathrm{SE}(3)}$ on the Lie algebra $\mathfrak{se}(3)$. To clarify this, firstly, we are looking for a projection mapping $\mathcal{A}$ and $\mathcal{D}$ to the same matrix in $\mathrm{SO}(3)$. As the convention in Eq. (\ref{eq:metric projection}) and Eq. (\ref{eq:metric projection2}), this is realized by a metric projection from the arithmetic mean $\frac{\mathcal{A}+\mathcal{D}}{2}$ to $\mathrm{Proj}_{\mathrm{SO}(3)}(\frac{\mathcal{A}+\mathcal{D}}{2})$. By Lem. \ref{lem:metric projection}, this is a metric projection, i.e. the projection gives a strict solution to the optimization problem: $\mathrm{Proj}_{\mathrm{SO}(3)}(\frac{\mathcal{A}+\mathcal{D}}{2})={\displaystyle \mathrm{arg}\min_{R\in\mathrm{SO}}}\|\frac{\mathcal{A}+\mathcal{D}}{2}-R\|_F:=R\ \!$. Therefore, we adopt it as the projective constraint for $\mathcal{A}$ and $\mathcal{D}$, and also for the orthogonal factor in $\mathcal{C}$.

As for $\mathcal{B}$, directly set the block to be zero matrix $O_{3\times 3}$ could preserve the manifold structure. The only problem left is how to find a well defined skew-symmetry representation $[\boldsymbol{t}_\Theta]^\wedge$ so that we could give a formula to calculate the projective constraint of $\mathcal{C}$, denoted by $\mathcal{C}^*$. Start from $\mathcal{C}^*=[\boldsymbol{t}_\Theta]^\wedge R$
since $R^{-1}=R^{\mathrm{T}}$, we have $[\boldsymbol{t}_\Theta]^\wedge=\mathcal{C^*}R^\mathrm{T}$. By Def. \ref{def:Skew-symmetric Representation}, $[\boldsymbol{t}_\Theta]^\wedge$ is skew-symmetric, we thereby take the skew-symmetric part of $\mathcal{C}R^\mathrm{T}$ to calculate: $    [\boldsymbol{t}_\Theta]^\wedge:=\frac{\mathcal{C}R^\mathrm{T}-(\ \!\mathcal{C}R^\mathrm{T})^\mathrm{T}}{2}=\frac{\mathcal{C}R^\mathrm{T}-R\ \!\mathcal{C^\mathrm{T}}}{2}$.\hfill$\square$

\section{Equilibrium Condition and Asymptotic Stability}\label{G}
\begin{theorem}[Equilibrium Condition and Asymptotic Stability]\label{Thm:EandU2}For the differential dynamical system of LieEDNN in Eq. (\ref{eq:LieEDNN3}), with $n=6N$, the existence and uniqueness of equilibrium are sufficient to be proven by following inequality on weight parameters, $L_i$ is the Lipschitz constant of activation $\phi_i(\cdot)$.
\begin{equation}
\sum_{i=1}^{n}L_i|w_{ij}|<\frac{\gamma}{\mu}
\label{eq:weight constraint 1 2}
\end{equation}
The equilibrium that is asymptotically stable can be established by the following inequality. $\mu$ and $\gamma$ are constants consistent with former notations. Eq. (\ref{eq:weight constraint 1}) and Eq. (\ref{eq:weight constraint 2}) can be satisfied by constraint on norm $\|W\|_1$ at the normalization process at learning stage.
\begin{equation}
    \sum_{i=1}^n(|w_{ji}|+|w_{ij}|)<\frac{2\gamma}{\mu}
    \label{eq:weight constraint 2 2}
\end{equation}
\end{theorem}
{\it\textbf{Proof:}} (i) \textbf{Existence of Equilibrium} Firstly, the activation functions selected are usually bounded. For this reason, we define a constant upper bound for all of the activation functions appearing in the below defined as $\phi\leq M$. We do not strictly demand that the activation function be uniformly continuous, but satisfying the Lipschitz condition. Out of a similar reason, we allocate each activation function a Lipschitz constant with respect to the subscript of notions as:
$|\phi_i(x)-\phi_i(y)|\le L_i|x-y|$.

To prove the existence of at least one critical point of the dynamic system in Eq. (\ref{eq:DNN}), which is to say when we replace the variable of the differential with zero, and by definition, every solution point of the variable functions $v_i$ counts. Here is to say the below equation system is solvable.
\begin{align}
    v_i = \frac{\mu}{\gamma}\left[ \sum_{j=1}^{n}w_{ij}\phi_j(v_j)+b_i\right]
\end{align}
established for $x\in\mathbb{R}$, $i=1,2,\dots n$. We denote the linear combination of symbols as matrix $A=\{\frac{\mu}{\gamma}w_{ij}\}_{n\times n}$ and bias or external current $\boldsymbol{\zeta}=\{\frac{\mu}{\gamma}b_i\}_{n\times 1}$. And vector multiple value function $\phi_i:\mathbb{R}\to\mathbb{R}$, define $\mathcal{F}:=\{ f|f:\mathbb{R}\to\mathbb{R}, \ f\ is\ Lipschitz\ continuous\}$, then $\boldsymbol{\phi}\in\mathcal{F}^n\otimes\mathbb{R}^n$ is to be: $\boldsymbol{\phi}(\boldsymbol{v})=\{\phi_i(v_i)\}_{n\times1}$. The system governed by Eq. (\ref{eq:DNN}) can then be rewritten in matrix form:
\begin{align}
    \boldsymbol{v}&=A\boldsymbol{\phi}(\boldsymbol{v})+\boldsymbol{\zeta}:=\boldsymbol{F}(\boldsymbol{v})
\end{align}
if we regard the right-hand side as a mapping $\boldsymbol{F}$ from $\mathbb{R}^n$ to $\mathbb{R}^n$, then the solution $\boldsymbol{v}$ can be treated as the fixed point of the mapping, allowing the analytic method in the proof.
\begin{align}
||\boldsymbol{F}(\boldsymbol{v})||^2&=\sum_{i=1}^{n} \left[\frac{\mu}{\gamma} \Bigg( \sum_{j=1}^{n}w_{ij}\phi_j(v_j)+b_i\Bigg)
\right]^2\\[0.1cm]
&\le\sum_{i=1}^{n} \frac{\mu^2}{\gamma^2}\Bigg[\left(\sum_{j=1}^{n}|w_{ij}|M\right)+|I_i|
\Bigg]^2\\[0.1cm]
&:=\rho^2
\end{align}
From $\rho$ we form a bounded convex set $\Omega=\{\boldsymbol{v}:||\boldsymbol{v}||\le\rho\}$, whereby Brouwer fixed point theorem applies to the proposition. Since $\boldsymbol{F}$ is a continuous mapping, there exists a $\boldsymbol{v^*}\in\Omega$ satisfying $\boldsymbol{F}(\boldsymbol{v^*})=\boldsymbol{v^*}$, which implies the existence of the critical point in the system governed by Eq. (\ref{eq:LieEDNN3}) and Eq. (\ref{eq:DNN}).
\hfill\(\square\)

{\it\textbf{Proof:}} (ii) \textbf{Uniqueness of Equilibrium} Through reduction to absurdity, we firstly suppose there exist two different critical points $\boldsymbol{v}^*,\ \boldsymbol{u}^*$. Consider $L_1$ norm of $v-u$ by:
\begin{align}
    \|\boldsymbol{u},\boldsymbol{v}\|_1&=\sum_{i=1}^{n}|u_i-v_i|=\sum_{i=1}^n\left|\frac{\mu}{\gamma} \sum_{j=1}^{n}w_{ij}[\phi_j(u_j)-\phi_j(v_j)]\right|\le\sum_{i,j=1}^n\frac{\mu}{\gamma}|w_{ij}| L_i|u_j-v_j|\\&=\sum_{i=1}^n\sum_{j=1}^{n}\frac{\mu}{\gamma}L_i|w_{ij}||u_j-v_j|=\sum_{j=1}^n\left(\sum_{i=1}^{n}\frac{\mu}{\gamma}L_i|w_{ij}|\right)|u_j-v_j|<\sum_{j=1}^{n}|u_i-v_i|
\label{equ23}
\end{align}
This inequality leads to a contradiction when the weights are restricted by 
\begin{align}
    \sum_{i=1}^{n}\frac{\mu}{\gamma}L_i|w_{ij}|<1
    \label{m2}
\end{align}
With the above constraints on weights, simply put, the sum of weights cannot be too large; then the network operation has a certain single convergent point for bounded input. The essential conclusion provides us with a theoretical base for complex network behaviour design and learning rules, which have not been clearly stated and summarized in the previous articles. The research on Hopfield neural network is relatively thorough, regardless of the stability or convergent pattern design. We transfer these basic theories and extend them to the construction of quaternion-valued network, in combination with knowledge of metric space and functional analysis. For simplification, when we use the inequality \ref{m2} for deeper deduction, we will suppose the coefficients in the front of weights to be a unit constant. And for the practical network operation, this condition will be satisfied during the procedure of weights normalization, where we set the norm of the weight matrix to be a constant and the significant information will not be lost.\hfill\(\square\)

{\it\textbf{Proof:}} (iii) \textbf{Asymptotic Stability} To judge the stability, we will be more specific here since there is possibility for the system to diverge to infinity, whereas merely being stable will not meet the whole requirement. The value of the output is expected to attach to a single point which, strictly speaking, is a critical point to be asymptotically stable, making the undesired situation such as chaos disappear.
\begin{align}
V(t)=\frac{1}{2\gamma}\sum_{j=1}^nx_j^2
\label{eq:acdd}
\end{align}
By Brouwer's fixed point theorem, we conclude the existence of a critical point, meaning \ref{s} established for any $j$ from 1 to n. By calculating the difference between the sum of the $L2$ metric of two critical points, we conclude the uniqueness of the critical point, where:
\begin{align}
    \dot{\boldsymbol{x}}=F_j(t,\boldsymbol{x}) = 0
    \label{s}
\end{align}
Suppose $Ki$ is the Lipschitz constant of the function $\phi_i(\boldsymbol{x})$, as the model of HNN is an autonomous system, we have $F_j(t,\boldsymbol{x})=F_j(\boldsymbol{x})$ for every $j=1,2,...,n$. Notice that:
\begin{align}
\frac{\text{d}V(t)}{\text{d}t}&=\frac{1}{\gamma}\sum_{j=1}^nx_j\dot x_j=\frac{1}{\gamma}\sum_{j=1}^nx_j\left[-\gamma x_j+\mu\sum_{i=1}^nw_{ji}\phi_i(x_i)\right]\\[0.1cm]
&\leq\sum_{j=1}^n\left(-{x_j}^2+\frac{\mu}{\gamma}\sum_{i=1}^n|w_{ji}|L_ix_ix_j\right)\\[0.1cm]
&\leq\sum_{j=1}^n\left[-{x_j}^2+\frac{\mu}{2\gamma}\sum_{i=1}^n|w_{ji}|L_i({x_i}^2+{x_j}^2)\right]\\[0.1cm]
&= \sum_{j=1}^n\left[-1+\frac{\mu}{2\gamma}\sum_{i=1}^n(|w_{ji}|+|w_{ij}|)\right]{x_j}^2\\[0.15cm]
&\leq 0
\label{B3}
\end{align}
By Lyapunov stability theory, the zero solution of the system is asymptotically stable such that the critical point of the original system is stable when the constraint of weights 
\begin{align}
    \frac{\mu}{2\gamma }\sum_{i=1}^n(|w_{ji}|+|w_{ij}|)<1
    \label{m1}
\end{align}
applies. Throughout the whole procedure, we notice that the activating function does not need to possess smoothness or continuity, but satisfy the Lipschitz condition.\hfill\(\square\)

\section{Smoothness of Trajectories Solved by Neural Dynamics}\label{H}
\begin{theorem}[Smoothness and Curvature Bound of LieEDNN Trajectories]
\label{thm:smoothness-curvature}
Consider the LieEDNN dynamics in Eq. (\ref{eq:LieEDNN3}): $\dot{\boldsymbol{\xi}}=-\gamma\boldsymbol{\xi}+\mu W\phi(\boldsymbol{\xi})+\mu\boldsymbol{b}$, where \(\boldsymbol{\xi}(t)\in\mathbb{R}^{6N}\), \(W\in\mathbb{R}^{6N\times 6N}\), and \(\phi\) acts component-wisely. Suppose \(\phi\in C^1\), \(\|J_\phi(\boldsymbol{\xi})\|_\infty\le L_\phi\), and the trajectory satisfies \(\|\boldsymbol{\xi}(t)\|_\infty\le R_\xi\) on the considered time interval. Then \(\boldsymbol{\xi}(t)\in C^2\), with convention $\|\boldsymbol{\xi}\|_\infty:=\xi_{\max}$ and $\|W\|_\infty:=\displaystyle\max_{1\le i\le6N}\Sigma_{j=1}^{6N}|w_{ij}|$, we have:
\begin{equation}
\|\ddot{\boldsymbol{\xi}}(t)\|_\infty
\le
\left(\gamma+\mu L_\phi\|W\|_\infty\right)
\left(
\gamma R_\xi+\mu\|W\|_\infty\|\phi(\boldsymbol{\xi}(t))\|_\infty+\mu\|\boldsymbol{b}\|_\infty
\right)
\end{equation}
In particular, if \(\phi=\tanh\), then \(L_\phi\le1\) and \(\|\phi(\boldsymbol{\xi}(t))\|_\infty\le1\), we have:
\begin{equation}
\|\ddot{\boldsymbol{\xi}}(t)\|_\infty
\le
\left(\gamma+\mu\|W\|_\infty\right)
\left(
\gamma R_\xi+\mu\|W\|_\infty+\mu\|\boldsymbol{b}\|_\infty
\right)
\end{equation}
Moreover, for each component curve \(t\mapsto \xi_i(t)\), its scalar curvature of $i$th entry of $\boldsymbol{\xi}$ satisfies:
\begin{equation}
\kappa_i(t)
:=
\frac{|\ddot{\xi}_i(t)|}{\left(1+|\dot{\xi}_i(t)|^2\right)^{3/2}}
\le
\|\ddot{\boldsymbol{\xi}}(t)\|_\infty<2\gamma(\gamma R_\xi+\gamma+\mu b_{\max}).
\end{equation}
\end{theorem}
\textit{\textbf{Proof:}} Define $F(\boldsymbol{\xi}):=-\gamma\boldsymbol{\xi}+\mu W\phi(\boldsymbol{\xi})+\mu\boldsymbol{b}$. Since \(\phi\in C^1\), the vector field \(F\) is \(C^1\). Therefore the solution satisfies \(\boldsymbol{\xi}(t)\in C^2\), denote $\ddot{\boldsymbol{\xi}}(t)=
J_F(\boldsymbol{\xi}(t))\dot{\boldsymbol{\xi}}(t)$, where,
\begin{equation}
J_F(\boldsymbol{\xi})
=\frac{\partial \ddot{\boldsymbol{\xi}}}{\partial \dot{\boldsymbol{\xi}}}=
-\gamma \mathbb{I}_{6N}
+
\mu WJ_\phi(\boldsymbol{\xi}).
\label{eq:1}
\end{equation}
Thus $\|J_F(\boldsymbol{\xi}(t))\|_\infty
\le
\gamma+\mu\|W\|_\infty\|J_\phi(\boldsymbol{\xi}(t))\|_\infty
\le
\gamma+\mu L_\phi\|W\|_\infty$.
Also,
\begin{equation}
\|\dot{\boldsymbol{\xi}}(t)\|_\infty
=
\|F(\boldsymbol{\xi}(t))\|_\infty
\le
\gamma R_\xi
+\mu\|W\|_\infty\|\phi(\boldsymbol{\xi}(t))\|_\infty
+\mu\|\boldsymbol{b}\|_\infty.
\label{eq:2}
\end{equation}
Combining the two inequalities in Eq. (\ref{eq:2}) and Eq. (\ref{eq:1}) gives
\begin{equation}
\|\ddot{\boldsymbol{\xi}}(t)\|_\infty
\le
\left(\gamma+\mu L_\phi\|W\|_\infty\right)
\left(
\gamma R_\xi+\mu\|W\|_\infty\|\phi(\boldsymbol{\xi}(t))\|_\infty+\mu\|\boldsymbol{b}\|_\infty
\right)
\label{eq:3}
\end{equation}
For \(\phi=\tanh\), we have \(L_\phi\le1\) and \(\|\phi(\boldsymbol{\xi}(t))\|_\infty\le1\), giving the stated special case.

By $(1+|\dot{\xi}_i(t)|^2)^{3/2}\ge1$, we have the curvature of $\xi_i(t)$ is bounded for $i=1,2,...,6N$ by:
\begin{equation}
\kappa_i(t)=\frac{|\ddot{\xi}_i(t)|}{\left(1+|\dot{\xi}_i(t)|^2\right)^{3/2}}\le|\ddot{\xi}_i(t)|\le\|\ddot{\boldsymbol{\xi}}(t)\|_\infty.
\end{equation}
With the bound of weight matrix norm $\|W\|_\infty<\gamma/\mu$ by Thm. \ref{Thm:EandU}, combine Eq. (\ref{eq:3}), we have further:
\begin{equation}
    \kappa_i(t)\le\|\ddot{\boldsymbol{\xi}}(t)\|_\infty<2\gamma(\gamma R_\xi+\gamma+\mu b_{\max})
\end{equation}
where $b_{\max}$ is the maximum of bias, in experiment, $b_i\equiv1/8$. Thus the theorem is proved.\hfill\(\square\)

\section{Sufficient Condition of Norm Constraints for Stable Global Equilibrium}\label{I}
\begin{theorem}[Sufficient Condition for Equilibrium Existence, Uniqueness and Asymptotic Stability]
\label{Thm:EandU3}
Consider the LieEDNN dynamics of Eq. (\ref{eq:LieEDNN3}), i.e. $
\dot{\boldsymbol{\xi}}
=
-\gamma\boldsymbol{\xi}
+\mu W\phi(\boldsymbol{\xi})
+\mu\boldsymbol b,
\ m=6N$. Each hyperbolic activation \(\phi_i\) is Lipschitz continuous with Lipschitz constant \(L_i\), we define $L_{\max}:=\max_{1\le i\le m}L_i=1$. By Thm. \ref{Thm:EandU} If:
\begin{equation}
\max_{1\le j\le m}\sum_{i=1}^{m}L_i|w_{ij}|<\frac{\gamma}{\mu},
\label{eq:weight constraint a}
\end{equation}
then the system has a unique equilibrium. A sufficient norm condition for this is:
\begin{equation}
\|W\|_1<\frac{\gamma}{\mu}.
\end{equation}
Moreover, if \(L_i\le 1\) and by Thm. \ref{Thm:EandU}, with condition on weight matrix:
\begin{equation}
\max_{1\le j\le m}\sum_{i=1}^{m}\bigl(|w_{ji}|+|w_{ij}|\bigr)<\frac{2\gamma}{\mu},
\label{eq:weight constraint b}
\end{equation}
the equilibrium is asymptotically stable. A sufficient norm condition for this is:
\begin{equation}
\|W\|_1+\|W\|_\infty<\frac{2\gamma}{\mu}.
\end{equation}
Further if $\|W\|_1, \|W\|_\infty<\gamma/\mu$, then equilibrium existence, uniqueness and asymptotic stability are established where we use conditions in Eq. (\ref{eq:weight constraint a}) and Eq. (\ref{eq:weight constraint b}).
\end{theorem}

\textit{\textbf{Proof:}}
The equilibrium equation is $\boldsymbol{\xi}
=
\frac{\mu}{\gamma}W\phi(\boldsymbol{\xi})
+
\frac{\mu}{\gamma}\boldsymbol b$. Define:
\begin{equation}
T(\boldsymbol{\xi})
:=
\frac{\mu}{\gamma}W\phi(\boldsymbol{\xi})
+
\frac{\mu}{\gamma}\boldsymbol b
\end{equation}
For any \(\boldsymbol{\xi},\boldsymbol{\eta}\in\mathbb R^m\), by Eq. (\ref{eq:weight constraint a}) we have:
\begin{equation}
\|T(\boldsymbol{\xi})-T(\boldsymbol{\eta})\|_1
\le
\frac{\mu}{\gamma}
\max_{1\le j\le m}
\sum_{i=1}^{m}L_i|w_{ij}|
\,
\|\boldsymbol{\xi}-\boldsymbol{\eta}\|_1<\infty
\end{equation}
which makes \(T\) a contraction. By Banach's fixed point theorem, the equilibrium exists and is unique by the proof of Thm \ref{Thm:EandU2}. Since $L_{\max}\le1$ and
\begin{equation}
\max_j\sum_i L_i|w_{ij}|
\le
L_{\max}\|W\|_1
\end{equation}
thus the norm condition \(\|W\|_1<\gamma/\mu\) is sufficient.

Let \(\boldsymbol{\xi}^*\) be the equilibrium and set $\boldsymbol e(t):=\boldsymbol{\xi}(t)-\boldsymbol{\xi}^*$
Then
\begin{equation}
\dot{\boldsymbol e}
=
-\gamma\boldsymbol e
+
\mu W\bigl[\phi(\boldsymbol{\xi})-\phi(\boldsymbol{\xi}^*)\bigr].
\end{equation}
Take $V(\boldsymbol e):=\frac12\|\boldsymbol e\|_2^2$, using \(L_i\le1\), we obtain
\begin{equation}
\dot V
=
-\gamma\|\boldsymbol e\|_2^2
+
\mu \boldsymbol e^{\mathrm T}W\bigl[\phi(\boldsymbol{\xi})-\phi(\boldsymbol{\xi}^*)\bigr]
\le
-\gamma\|\boldsymbol e\|_2^2
+
\mu |\boldsymbol e|^{\mathrm T}|W||\boldsymbol e|.
\end{equation}
Since $
|\boldsymbol e|^{\mathrm T}|W||\boldsymbol e|
=
|\boldsymbol e|^{\mathrm T}\frac{|W|+|W|^{\mathrm T}}{2}|\boldsymbol e|$, with inequality $
|\boldsymbol e|^{\mathrm T}|W||\boldsymbol e|
\le
\frac12
\max_j\sum_i\bigl(|w_{ji}|+|w_{ij}|\bigr)
\|\boldsymbol e\|_2^2$. Therefore we have under Eq.~\eqref{eq:weight constraint 2}, the derivative of generalized energy function:
\begin{equation}
\dot V
\le
-\left[
\gamma
-
\frac{\mu}{2}
\max_j\sum_i\bigl(|w_{ji}|+|w_{ij}|\bigr)
\right]
\|\boldsymbol e\|_2^2
<0
\end{equation}
for \(\boldsymbol e\ne0\). Hence the equilibrium is asymptotically stable by Lyapunov theorem \cite{Lyapunov,LYAPUNOVtheory}. Finally, by definition of matrix norms $\|W\|_\infty:=\displaystyle\max_{1\le i\le6N}\Sigma_{j=1}^{6N}|w_{ij}|$ and $\|W\|_1:=\displaystyle\max_{1\le j\le6N}\Sigma_{i=1}^{6N}|w_{ij}|$:
\begin{equation}
\max_j\sum_i\bigl(|w_{ji}|+|w_{ij}|\bigr)
\le
\|W\|_1+\|W\|_\infty,
\end{equation}
so \(\|W\|_1+\|W\|_\infty<2\gamma/\mu\) is a sufficient norm condition for the establishments of equilibrium existence, uniqueness and asymptotic stability, so does $\|W\|_1, \|W\|_\infty<\gamma/\mu$ by Eq. (\ref{eq:weight constraint a}) and q. (\ref{eq:weight constraint b}).
\hfill\(\square\)

\section{Limitations and Future Works}\label{J}
\paragraph{Limitations} (i) The current LieEDNN framework is evaluated in a prototype telescopic manipulator planning task with $\mathrm{SE}(3)$ and $\mathfrak{se}(3)$. Although the experiments show that the learned dynamics can converge to the desired equilibrium and preserve the Lie group induced block structure under periodic projection, the setting is still limited in practical complexity. The present implementation does not yet include collision avoidance, joint limits, and safety constraints, which are necessary for the actual robotic applications. (ii) The stability requirement is given by sufficient conditions on the weight matrix, but how this constraints together with manifold projection will influence the convergence and stability of training in loss function optimization is not rigorously analyzed. (iii) Periodic metric projection preserves the geometric structure empirically, but it may introduce oscillations during training and requires careful selection of the projection periods. How does this approach perform in comparison with the Riemann gradient is not fully explored.
\paragraph{Future Works} Aiming to address the limitations mentioned above, the future research contents involve: (i) Impose practical robotic constraints into the LieEDNN framework, including collision avoidance, joint limits, and safety of trajectory generation for realistic robotic manipulation tasks. (ii) Rigorous theoretical analysis of the interaction between stability constraints, manifold projection, and optimization dynamics, particularly the convergence properties of the learning process under constrained Lie group structured parameter spaces. (iii) Investigation of alternative geometry preserving optimization approaches, such as Riemann gradient based methods, together with systematic comparison against periodic projection strategies in terms of convergence, stability, and computational efficiency. (iv) Extension and implementation of the proposed framework on Lie groups beyond $\mathrm{SE}(3)$, including more general matrix Lie groups arising in robotics, control systems, and geometric learning problems.
\section{Demonstration of Representative Blocks for Lie Group $\mathrm{SE}(3)$ Embedding}\label{K}
For system defined in Sec. \ref{sec:architecture}, i.e. network evolution of LieEDNN, we have the following two equivalent forms:
\begin{align}
&\frac{\text{d}{}}{\text{d}t}{\xi}_i=-\gamma{\xi}_i+\mu \sum_jw_{ij}\phi({\xi_i})+\mu{b}_i,\quad \xi_i\in\mathbb{R},\quad i=1,2,\dots6N\\
    \frac{\text{d}{}}{\text{d}t}{\boldsymbol{\xi}}_i&=-\gamma \boldsymbol{\xi}_i+\mu\sum_j \alpha_{ij}\mathcal{L}[\mathcal{A}d_{\theta_{ij}}]\phi(\boldsymbol{\xi_j})+\mu \boldsymbol{b_i},\quad \boldsymbol{\xi}_i\in\mathbb{R}^6\quad i=1,2,...,N
    \label{eq:Lie}
\end{align}
And the explicit expression of weight matrix has three equivalent forms as below:
\begin{align}
    W&=\left[\begin{matrix}
    {w}_{1,1}&{w}_{1,2}&\dots&{w}_{1,6N}\\
    {w}_{2,1}&{w}_{2,2}&\dots&{w}_{2,6n}\\
    \vdots&\vdots&\ddots&\vdots\\
    {w}_{6N,1}&{w}_{n2}&\dots&{w}_{6N,6N}
    \end{matrix}\right]=\left[\begin{matrix}
    W_{1,1}&W_{1,2}&\dots&W_{1,N}\\
    W_{2,1}&W_{2,2}&\dots&W_{2,N}\\
    \vdots&\vdots&\ddots&\vdots\\
    W_{N,1}&W_{N,2}&\dots&W_{N,N}
    \end{matrix}\right]\\[0.2cm]
    &=\left[\begin{matrix}
    \alpha_{1,1}\mathcal{L}[\mathcal{A}d_{\theta_{1,1}}]&\alpha_{1,2}\mathcal{L}[\mathcal{A}d_{\theta_{1,2}}]&\dots&\alpha_{1.N}\mathcal{L}[\mathcal{A}d_{\theta_{1,N}}]\\
    \alpha_{2,2}\mathcal{L}[\mathcal{A}d_{\theta_{2,2}}]&\alpha_{2,2}\mathcal{L}[\mathcal{A}d_{\theta_{2,2}}]&\dots&\alpha_{2,N}\mathcal{L}[\mathcal{A}d_{\theta_{2,N}}]\\
    \vdots&\vdots&\ddots&\vdots\\
    \alpha_{N,1}\mathcal{L}[\mathcal{A}d_{\theta_{N,1}}]&\alpha_{N,2}\mathcal{L}[\mathcal{A}d_{\theta_{N,2}}]&\dots&\alpha_{N,N}\mathcal{L}[\mathcal{A}d_{\theta_{N,N}}]
    \end{matrix}\right]
\label{compatible form}
\end{align}

\section{Decoding the Motion Trajectory for Manipulator by Manifold Retraction}\label{L}

The neural network dynamics are not decoded directly as joint states, but as a trajectory on the Lie algebra $\mathfrak{se}(3)$. Let $\xi(t)$ denote the network state at time $t$, and let $
\xi(t)=
\begin{bmatrix}
\omega(t),
v(t)
\end{bmatrix}^{\mathrm{T}}
\in\mathbb{R}^{6}$
be the decoded twist coordinate, where \(v(t)\in\mathbb{R}^{3}\) is the translational component, and \(\omega(t)\in\mathbb{R}^{3}\) is the rotational component. The corresponding Lie algebra element is defined by:
\begin{equation}
\widehat{\xi}(t)
=
\begin{bmatrix}
\omega^\wedge(t) & v(t)\\
O_{1\times 3} & 0
\end{bmatrix}
\in\mathfrak{se}(3)
\end{equation}
where $\mathcal{P}(\widehat{\xi}(t))=\xi(t)$. The manipulator pose is then obtained by retracting this algebra element onto the Lie group through the exponential map:
\begin{equation}
T(t)=
\mathrm{Exp}(\widehat{\xi}(t))
\in \mathrm{SE}(3).
\end{equation}
Writing \(\theta(t)=\|\omega(t)\|\) to be the rotation angle, the closed-form expression of exponential retraction is:
\begin{equation}
T(t)
=
\begin{bmatrix}
R(t) & p(t)\\
O_{1\times 3} & 1
\end{bmatrix},
\quad
R(t)=\mathrm{Exp}(\omega^\wedge(t)),
\quad
p(t)=J[\omega(t)]v(t)
\end{equation}
where \(J[\omega(t)]\) is the left Jacobian of \(\mathrm{SO}(3)\), given by
\begin{equation}
J(\omega)
=
\mathbb{I}_{3}
+
\frac{1-\cos\theta}{\theta^{2}}{\omega}^\wedge
+
\frac{\theta-\sin\theta}{\theta^{3}}[{\omega}^\wedge]^{2},
\quad
\theta=\|\omega\|
\end{equation}
For small \(\theta\), the same expression is evaluated by its Taylor expansion to avoid numerical singularity:
\begin{equation}
J(\omega)
=
\mathbb{I}_{3}
+\frac{1}{2}{\omega}^\wedge
+\frac{1}{6}[{\omega}^\wedge]^2
+O(\theta^{3})
\end{equation}
Therefore, the continuous-time neural dynamics generate a smooth curve \(\xi(t)\in\mathfrak{se}(3)\), while the exponential retraction produces a physically valid rigid-body motion \(T(t)\in\mathrm{SE}(3)\). For a telescopic manipulator, the translational part \(p(t)\) gives the decoded joint displacement, while the rotational part \(R(t)\) gives the decoded joint orientation. Since \(T(t)\) always lies in \(\mathrm{SE}(3)\), the decoded trajectory preserves the rigid-body constraint \(R(t)^{\mathrm{T}}R(t)=\mathbb{I}_{3}\) and \(\det [R(t)]=1\), instead of relying on an unconstrained Euclidean output followed by post hoc normalization. This follows the standard formulation of rigid-body motion and the exponential mapping from \(\mathfrak{se}(3)\) to \(\mathrm{SE}(3)\) used in geometric robotics \cite{Robot2,stateestimate,SE3tutorial}.

\section{Experimental Verification of LieEDNN}\label{M}
\paragraph{Learning Algorithm}We provide the algorithm for the learning process of LieEDNN on random targets using in experiments of Sec. \ref{sec:experiment} of main texts. All presetting and hyperparameters are also provided in the initialization part of Alg. \ref{alg.1}. The dataset of test targets are randomly generated to avoid statistical impact. \textbf{Compute Resources for Experiment} All experiments reported in this paper were conducted on a local personal computer (PC) laptop using CPU only. The experiments were run on a Mac laptop equipped with an Apple M1 chip and 8 GB unified memory. The reported experiments consist of several runs over different projection periods and ablation when manifold projection is canceled in training. The full set of experiments can be reproduced on a standard personal laptop with comparable CPU and memory resources. A complete training process takes three to five minutes without GPU acceleration. No additional large scale preliminary experiments requiring substantially larger compute resources for the reported results.\\[0.1cm]
\begin{algorithm}[H]
\caption{Projective Manifold Learning of LieEDNN}
\KwIn{Target $\boldsymbol{\xi}_d\in\mathbb{R}^{6N}$, obeying ${\xi}_d^i\thicksim\mathcal{U}(-0.8,0.8)$. Number of neurons$\ =N$.}
\KwOut{Trained weight matrix $W\in\mathbb{R}^{6N\times6N}$. Loss and accuracy curves about epochs.}
\textbf{Initialization:} \justifying Range of adaptive learning rates $\eta\in[0.002,0.2]$. Projection period $\mathscr{P}=10$.  Error tolerance $\tau=10^{-5}$. Bias $b=0.125\sim0.2$. Maximum epochs $T_{\max}=25000$. Network parameters $\mu=\gamma=1$. Lipschitz constant of activation $L_i=1$. Weights initialization $w_{ij}\sim\mathcal{U}(-1,1)$. Weights normalization through $W = W/(\varepsilon+\|W\|_1)$, where $\varepsilon=10^{-12}\ll 1$.
\BlankLine
\For(\tcp*[f]{Training Iteration Loop}){$k=1$ \KwTo $T_{\max}$}{
    $\frac{\mathrm{d}}{\mathrm{d}t}\boldsymbol{\xi}=-\gamma\boldsymbol{\xi}+ \mu\,{W}{\phi}(\boldsymbol{\xi})+\mu\boldsymbol{b}$\tcp*[f]{Solve by Runge-Kutta45 Method}\\[0.1cm]
    \hspace*{-0.15cm}$\boldsymbol{\delta} \leftarrow \boldsymbol{\xi}^*-\boldsymbol{\xi}_d$\tcp*[f]{Difference to Target and the Loss}\\ \hspace*{-0.15cm}$E\leftarrow\|\boldsymbol{\xi}^*-\boldsymbol{\xi}_d\|_2$\\
    \hspace*{-0.15cm}\For{$i=1,\dots,6N$}{
    \For{$j=1,\dots,6N$}{
    \vspace*{0.1cm}
    $S=\mathbb{I}_{6N}-\frac{\mu}{\gamma}W\cdot J_{\phi}(\boldsymbol{\xi^*})$\tcp*[f]{Calculate Sensitivity Matrix}\\
     \hspace*{-0.15cm}$w_{ij}^{(k+1)} \leftarrow w_{ij}^{(k)}-\eta\frac{\mu\phi(\xi_j^*)}{\gamma E}\boldsymbol{\delta}^{\mathrm{T}}S^{-1}\boldsymbol{e}_i$\tcp*[f]{Gradient Descent Eq.(\ref{eq:gradient})}
    }
  }
  \If(\tcp*[f]{Periodic Manifold Projection}){$k\equiv0\ (mod\ \mathscr{P})$}{
  \For{$i=1,\dots,N$}{
    \For{$j=1,\dots,N$}{
  $\mathcal{L}[\Ad_{\theta_{ij}}]\leftarrow\mathrm{Proj}_\mathscr{L}(W_{ij}/\alpha_{ij})$\tcp*[f]{Block-wise projection Eq.(\ref{eq:metric projection})}\\\hspace*{-0.15cm}$W_{ij}=\alpha_{ij}\mathcal{L}[\Ad_{\theta_{ij}}]$\tcp*[f]{Reconstruct Representation Eq.(\ref{eq:LieEDNN2})}
        }
    }  
  }
  \If(\tcp*[f]{Training Stop Criteria}){$\|\boldsymbol{\delta}\|_{\infty}<\tau$ $\boldsymbol{\mathrm{and}}$ $\mathscr{P|}k$ $\boldsymbol{\mathrm{and}}$ Accuracy=1.0}{\textbf{break}\;
    }
}
\label{alg.1}
\end{algorithm} 
\section{Group Equivariance of LieEDNN}\label{N}
\paragraph{Sufficient Condition for Group Equivariance} This appendix discusses a sufficient modification under which LieEDNN becomes equivariant with respect to a Lie group action. Let $G$ be matrix Lie group adjoint action, let $\mathfrak{g}$ be its Lie algebra. It should be noted that a component-wise scalar activation, such as applying hyperbolic tangent \(\tanh(\cdot)\) to each coordinate of \(X\), is generally not equivariant under this action. In general,
\begin{equation}
\tanh(g\circ X\circ g^{-1}) \neq g\circ\tanh(X)\circ g^{-1}
\end{equation}
Therefore, the standard component-wise activation does not by itself guarantee Lie group equivariance. A sufficient way to obtain equivariance is to replace the component-wise activation by equivariant function. For every \(g\in G\), and adjoint action $\mathcal{A}d_g(\cdot):=g(\cdot)g^{-1}$. One has the following equation as a condition.
\begin{equation}
\phi(\mathcal{A}d_g\circ X)
=
\phi(g\circ X\circ g^{-1})
=
g\circ\phi(X)\circ g^{-1}
=
\mathcal{A}d_g\circ\phi(X).
\end{equation}
Thus, \(\phi\) is equivariant under the adjoint action. Before we give a activation satisfying this condition, we firstly elaborate why the dynamics of LieEDNN is equivariant under $\mathcal{A}d_g(\cdot)$. Consider the LieEDNN dynamics
\begin{equation}
\frac{\mathrm{d}}{\mathrm{d}t}\boldsymbol{\xi}_i=-\gamma\boldsymbol{\xi}_i+\mu\sum_{j}\alpha_{ij}\mathcal{L}[\mathcal{A}d_{\theta_{ij}}]\phi(\boldsymbol{\xi}_j)+\mu \boldsymbol{b}_i,\quad i=1,2,...,N
\end{equation}
With Lie group action on the Lie group weights, the weights under new transformed references are
\begin{align}
    \theta_{ij}\rightsquigarrow\tilde{\theta}_{ij}:=&h\tilde{\theta}_{ij} h^{-1}\ \Longrightarrow\ \mathcal{L}[\mathcal{A}d_{\tilde{\theta}_{ij}}]=\mathcal{L}[\mathcal{A}d_h]\mathcal{L}[\mathcal{A}d_{{\theta}_{ij}}]\mathcal{L}[\mathcal{A}d_{h^{-1}}]
\end{align}
and Lie group action on biases, the biases under new transformed references are
\begin{align}
\boldsymbol{b}_i\rightsquigarrow \tilde{\boldsymbol{b}}_i:=h\circ\boldsymbol{b}_i\circ h^{-1}=\mathcal{L}[\mathcal{A}d_h]\boldsymbol{b}_i
\end{align}
The dynamics under Lie group action is expressed by
\begin{align}
    \frac{\mathrm{d}}{\mathrm{d}t}\tilde{\boldsymbol{\xi}}_i&=-\gamma\tilde{\boldsymbol{\xi}}_i+\mu\sum_{j}\alpha_{ij}\mathcal{L}[\mathcal{A}d_{\tilde{\theta}_{ij}}]\phi(\tilde{\boldsymbol{\xi}}_j)+\mu \tilde{\boldsymbol{b}}_i\\
    &=-\gamma\mathcal{L}[\mathcal{A}d_h]\boldsymbol{\xi}_i+\mu\sum_{j}\alpha_{ij}\mathcal{L}[\mathcal{A}d_{h\theta_{ij}h^{-1}}]\phi(\mathcal{L}[\mathcal{A}d_h]\boldsymbol{\xi}_j)+\mu \mathcal{L}[\mathcal{A}d_h]\boldsymbol{b}_i\\
    &=-\gamma\mathcal{L}[\mathcal{A}d_h]\boldsymbol{\xi}_i+\mu\mathcal{L}[\mathcal{A}d_h]\sum_{j}\alpha_{ij}\mathcal{L}[\mathcal{A}d_{{\theta}_{ij}}]\phi(\boldsymbol{\xi}_j)+\mu \mathcal{L}[\mathcal{A}d_h]\boldsymbol{b}_i\\
    &=\mathcal{L}[\mathcal{A}d_h]{\large(}\!-\!\gamma\boldsymbol{\xi}_i+\mu\sum_{j}\alpha_{ij}\mathcal{L}[\mathcal{A}d_{{\theta}_{ij}}]\phi(\boldsymbol{\xi}_j)+\mu\boldsymbol{b}_i{\large)}\\
    &=\mathcal{L}[\mathcal{A}d_h]\frac{\mathrm{d}}{\mathrm{d}t}{\boldsymbol{\xi}}_i
\end{align}
With the same action on the initialization
\begin{equation}
    \tilde{\boldsymbol{\xi}}_i(0):=\mathcal{L}[\mathcal{A}d_h]{\boldsymbol{\xi}}_i(0)
\end{equation}
we have
\begin{equation}
    \tilde{\boldsymbol{\xi}}_i(t)=\tilde{\boldsymbol{\xi}}_i(0)+\int_0^t\frac{\mathrm{d}}{\mathrm{d}t}\tilde{\boldsymbol{\xi}}_i(\tau)\mathrm{d}\tau=\mathcal{L}[\mathcal{A}d_h]\left\{{\boldsymbol{\xi}}_i(0)+\int_0^t\frac{\mathrm{d}}{\mathrm{d}t}{\boldsymbol{\xi}}_i(\tau)\mathrm{d}\tau\right\}=\mathcal{L}[\mathcal{A}d_h]{\boldsymbol{\xi}}_i(t)
\end{equation}
A satisfactory choice of activation for $\mathfrak{se}(3)$ is:
\begin{equation}
    \Phi(X)=\frac{\mathrm{tanh}(\|\boldsymbol{\omega}\|)}{\|\boldsymbol{\omega}\|+\epsilon}X
\end{equation}
For general matrix Lie group and Lie algebra, consider:
\begin{align}
    \text{tr}\left[\left(\mathcal{A}d_g\circ X\right)^2\right]&=\text{tr}\left[\left(g\circ \!X\!\circ g^{-1}\right)^2\right]\\
    &=\text{tr}\left[\left(g\circ \!X\!\circ g^{-1}g\circ \!X\!\circ g^{-1}\right)\right]\\
    &=\text{tr}\left[\left(g\circ \!X^2\!\circ g^{-1}\right)\right]\\
    &=\text{tr}\left[\left(X^2\!\circ g^{-1}g\right)\right]\\
    &=\text{tr}[X^2]
\end{align}
Thus $\text{tr}(X^2)$ is a equivariant function of $X$ under adjoint action $\mathcal{A}d_g$, where $g\in G$ is a fixed Lie group element. Notice that for $X\in \mathfrak{se}(3)$, there is
\begin{equation}
X=\left[\begin{matrix}
[\boldsymbol{\omega}]^\wedge & \boldsymbol{v}\\
O_{1\times 3} & 0 
\end{matrix}\right],\quad X^2=\left[\begin{matrix}
{[\boldsymbol{\omega}]^\wedge}^2 & [\boldsymbol{\omega}]^\wedge\boldsymbol{v}\\
O_{1\times 3} & 0 
\end{matrix}\right]
\end{equation}
and
\begin{equation}
\text{tr}(X^2)=\text{tr}({[\boldsymbol{\omega}]^\wedge}^2)=\text{tr} \left( \left[ \begin{matrix} - \omega _{3}^{2} - \omega _{2}^{2} & \omega _{1} \omega _{2} & \omega _{3} \omega _{1} \\ \omega _{1} \omega _{2} & - \omega _{3}^{2} - \omega _{1}^{2} & \omega _{3} \omega _{2} \\ \omega _{3} \omega _{1} & \omega _{2} \omega _{3} & - \omega _{2}^{2} - \omega _{1}^{2} \end{matrix} \right] \right)=-2\|\boldsymbol{\omega}\|^2
\end{equation}
Thus the proposed function $\Phi$ is a function of $\text{tr}(X^2)$ therefore is also equivariant under Lie group adjoint action. By the following relation, the sufficient condition of Lie group equivariance for LieEDNN is satisfied. 
\begin{equation}
    \Phi(X)=\frac{\mathrm{tanh}\left(-\sqrt{\text{tr}(X^2)}/2\right)}{-\sqrt{\text{tr}(X^2)}/2+\epsilon}X=\Phi\left(\text{tr}(X^2)\right),\quad \Phi(\mathcal{A}d_g\circ X)=\mathcal{A}d_g\circ\Phi(X) 
\end{equation}

This result is significant because equivariance encodes the symmetry of the underlying geometric space directly into the neural dynamics. Instead of learning equivalent behaviours separately under different coordinate frames, the model preserves the transformation law by construction. This is consistent with the principle of group equivariant neural networks, where symmetry-aware architectures exploit known group structure to improve generalization and reduce redundant learning \cite{cohen2016group}. It is also aligned with the geometric deep learning perspective, where equivariance provides a general mechanism for incorporating geometric priors into neural architectures \cite{bronstein2021geometric}. For robotic motion planning on Lie groups, such as \(\mathrm{SE}(3)\), adjoint equivariance is particularly meaningful because rigid-body motions and changes of coordinate frames are naturally described by Lie group actions. Hence, an equivariant LieEDNN variant can improve geometric consistency, interpretability, and structural reliability of the learned dynamics. Related Liegroup equivariant architectures, such as LieConv, further support the value of imposing Lie group equivariance when learning on continuous geometric and dynamical data \cite{finzi2020generalizing}.
%%%%%%%%%%%%%%%%%%%%%%%%%%%%%%%%%%%%%%%%%%%%%%%%%%%%%%%%%%%%

% \newpage
% \input{checklist.tex}

\end{document}